\documentclass{article}
\usepackage[nonatbib,final]{neurips_2021}
\usepackage[numbers]{natbib}
\usepackage[utf8]{inputenc} 
\usepackage[T1]{fontenc}    
\usepackage{hyperref}       
\usepackage{url}            
\usepackage{booktabs}       
\usepackage{amsfonts}       
\usepackage{nicefrac}       
\usepackage{microtype}      

\usepackage{mdframed}
\usepackage{amsmath,mathrsfs,amssymb,mathtools,bm}
\usepackage{scalerel}
\usepackage{amsthm}
\usepackage{txfonts}
\usepackage{fontawesome}
\usepackage{wrapfig}
\usepackage{todonotes}
\usepackage{xcolor}
\usepackage{setspace}
\usepackage{multirow}
\usepackage{bm,bbm}
\usepackage{paralist}
\usepackage{enumitem}
\usepackage{caption}[small]
\usepackage{subcaption} 

\newtheorem{assumption}{Assumption}
\newtheorem{corollary}{Corollary}
\newtheorem{problem*}{Problem}

\newtheorem{theorem}{Theorem}
\newtheorem{lemma}{Lemma}

\newtheorem{definition}{Definition}

\newtheorem{example}{Example}
\newtheorem*{example*}{Example}
\usepackage[ruled,noresetcount,vlined]{algorithm2e}
\makeatletter

\newcommand{\nosemic}{\renewcommand{\@endalgocfline}{\relax}}
\newcommand{\dosemic}{\renewcommand{\@endalgocfline}{\algocf@endline}}
\let\oldnl\nl
\newcommand{\nonl}{\renewcommand{\nl}{\let\nl\oldnl}}

\SetKwFor{Repeat}{repeat}{:}{endw}
\makeatother

\usepackage{todonotes}
\usepackage{xcolor}
\usepackage{setspace}
\usepackage{multirow}
\usepackage{bm,bbm}
\usepackage{paralist}

\DeclareMathOperator*{\argmin}{argmin}

\newcommand{\clip}{\textsl{clip}}
\newcommand{\noise}{\textsl{noise}}

\definecolor{darkgreen}{RGB}{204,102,0}

\newcommand{\prune}[1]{{\color{gray}{#1}}}
\newcommand*{\defeq}{\stackrel{\text{def}}{=}}

\newcommand{\cA}{\mathcal{A}}

 \newcommand{\cL}{\mathcal{L}}
\newcommand{\cM}{\mathcal{M}} \newcommand{\cN}{\mathcal{N}}

 \newcommand{\cX}{\mathcal{X}}
\newcommand{\EE}{\mathbb{E}} \newcommand{\RR}{\mathbb{R}}

\newcommand{\btheta}{{\bm{\theta}}}
\newcommand{\bg}{\bm{g}}
\newcommand{\bH}{\bm{H}}

\newcommand{\var}{\mathrm{Var}}

\newcommand\norm[1]{\left\lVert#1\right\rVert}
\newcommand\hugA[1]{\left\langle#1\right\rangle}
\newcommand\hugB[1]{\left[{#1}\right]}
\newcommand\hugP[1]{\left(#1\right)}
\newcommand\abs[1]{\left|#1\right|}
\DeclareMathOperator{\Tr}{Tr}

\title{Differentially Private Empirical Risk Minimization under the Fairness Lens}

\author{%
  Cuong Tran\\
  Syracuse University\\
  \texttt{cutran@syr.edu} \\
  \And
  My H.~Dinh\\
  Syracuse University\\
  \texttt{mydinh@syr.edu}\\
  \And
  Ferdinando Fioretto\\
  Syracuse University\\
  \texttt{ffiorett@syr.edu}\\
}

\begin{document}

\maketitle

\begin{abstract}
Differential Privacy (DP) \cite{dwork:06} is an important privacy-enhancing technology for private machine learning systems. It allows to measure and bound the risk associated with an individual participation in a computation. However, it was recently observed that DP learning systems may  exacerbate bias and unfairness for different groups of individuals \cite{NEURIPS2019_eugene,xu2020removing,pujol:20}. This paper builds on these important observations and sheds light on the causes of the disparate impacts arising in the problem of differentially private empirical risk minimization. It focuses on the accuracy disparity arising among groups of individuals in two well-studied DP learning methods: output perturbation \cite{chaudhuri2011differentially} and differentially private stochastic gradient descent \cite{abadi2016deep}. The paper analyzes which data and model properties are responsible for the disproportionate impacts, why these aspects are affecting different groups disproportionately, and proposes guidelines to mitigate these effects. The proposed approach is evaluated on several datasets and settings.
\end{abstract}

\section{Introduction}
\label{sec:introduction}

{While learning systems have become instrumental for many decisions 
and policy operations involving individuals, the use of rich datasets 
combined with the adoption of black-box algorithms has sparked concerns 
about how these systems operate. 
Two key concerns regard how these systems handle discrimination and 
how much information they leak about the individuals whose data is used as input.}

Differential Privacy (DP) \cite{dwork:06} has become the paradigm of
choice for protecting data privacy and its deployments are growing at
a fast rate. 
DP is appealing as it bounds the risks of disclosing
sensitive information of individuals participating in a computation.
However, it was recently observed that DP systems may induce biased and unfair outcomes for different groups of individuals 
\cite{NEURIPS2019_eugene,pujol:20,xu2020removing}. 
The resulting outcomes can have significant societal and economic
impacts on the involved individuals: classification errors may
penalize some groups over others in important determinations
including criminal assessment, landing, and hiring 
\cite{NEURIPS2019_eugene} or can result in disparities regarding
the allocation of critical funds, benefits, and therapeutics \cite{pujol:20}.
{\em While these surprising observations have become apparent in
several contexts, their causes are largely understudied and not 
fully understood.}

This paper makes a step toward addressing this important knowledge 
gap. It builds on these key observations and sheds light on
the causes of the disparate impacts arising in the problem of
differentially private empirical risk minimization (ERM). It focuses
on the {accuracy disparity} arising among groups of individuals in two
well-studied DP learning methods: output perturbation
\cite{chaudhuri2011differentially} and differentially private
stochastic gradient descent (DP-SGD) \cite{abadi2016deep}. The paper
analyzes which properties of the model and the data are responsible
for the disproportionate impacts, why these aspects are affecting
different groups disproportionately, and proposes guidelines to
mitigate these effects. 

In summary, the paper makes the following contributions:
  \begin{enumerate}[leftmargin=*, parsep=0pt, itemsep=2pt, topsep=-4pt]
     \item It develops a notion of fairness under private training that relies on the concept of excessive risk.

     \item It analyzes this fairness notion in two DP learning 
     methods: output perturbation and DP-SGD. 

     \item It isolates the relevant components related with noise addition and gradient clipping responsible for the disparate impacts.

     \item It studies the behaviors and the causes for these components 
     to affect different groups of individuals disproportionately during private training. 

     \item Based on these observations, it proposes a mitigation solution and evaluates its effectiveness on several standard datasets.
  \end{enumerate}
To the best of the authors knowledge, this work represents a first step toward a deeper understanding of the causes of the unfairness impacts in differentially private learning. 

\section{Related work}
\label{sec:related_work}
The research at the interface between differential privacy and 
fairness is receiving increasing attention and can be broadly 
categorized into three main lines of work.
The first shows that DP is in alignment with fairness. 
Notable contribution in this direction include \citet{dwork2012fairness} seminal work, which highlights the relation 
between individual fairness and differential privacy, and 
\citet{mahdi2020improving}, which shows that the private exponential 
mechanism can produce fair outcomes in some selection problems. 
Works in the second category study the setting under which a fair model 
can leak privacy 
\cite{mozannar2020fair, jagielski2019differentially, chang2020privacy,tran2020differentially,Xu_2019}. 
These works propose learning frameworks that guarantee DP while also 
encouraging the satisfaction of different notions of fairness. 
{For example, \citet{xu2020removing} proposes a private and fair variant of DP-SGD that uses separate clipping bounds for each groups of individuals. Such proposal encourages accuracy parity at the expense of an extra privacy cost (required to customize the clipping bound for each group).}
Works in the last category show that private mechanisms can have a 
negative impact towards fairness \cite{pujol:20, xu2020removing, 
NEURIPS2019_eugene, farrand2020neither, tran2020differentially}.
For example, Cummings et al.~\cite{cummings2019compatibility} shows 
that it is impossible to achieve \emph{exact} equalized odds while 
also satisfying pure DP. 
\citet{pujol:20} observe that decisions made using a 
private version of a dataset may disproportionately affect some 
groups over others. 
Similar observations were also made in the context of model learning. \citet{NEURIPS2019_eugene} empirically observed 
that the accuracy of a DP model trained using DP-SGD drops 
disproportionately across groups causing larger negative impacts to 
the underrepresented groups. 
\citet{farrand2020neither} reaches similar conclusions. The authors empirically show that the disparate impact of differential privacy on model accuracy is not limited to highly imbalanced data and can occur even in situations where the classes are slightly imbalanced. 

This paper builds on this body of work and their important 
empirical observations. It derives the conditions and studies the 
causes of unfairness in the context of private empirical risk minimization 
problems as well as it introduces mitigating guidelines.

\section{Preliminaries}
\label{sec:preliminaries}
 Differential privacy (DP) \cite{dwork:06} is a strong privacy notion used to quantify and bound the privacy loss of an individual participation to a computation. 
 Informally, it  states that the probability of any output does not change much when a record is added or removed from a dataset, limiting the amount of information that the output reveals about any individual.  
The action of adding or removing a record from a dataset $D$, resulting in a new dataset $D'$, defines the notion of \emph{adjacency}, denoted $D \sim D'$.
\begin{definition}
  \label{dp-def}
  A mechanism $\cM \!:\! \mathcal{D} \!\to\! \mathcal{R}$ with domain $\mathcal{D}$ and range $\mathcal{R}$ is $(\epsilon, \delta)$-differentially private, if, for any two adjacent inputs $D \sim D' \!\in\! \mathcal{D}$, and any subset of output responses $R \subseteq \mathcal{R}$:
  \[
      \Pr[\cM(D) \in R ] \leq  e^{\epsilon} 
      \Pr[\cM(D') \in R ] + \delta.
  \]
\end{definition}
\noindent 
Parameter $\epsilon > 0$ describes the \emph{privacy loss} of the algorithm, with values close to $0$ denoting strong privacy, while parameter 
$\delta \in [0,1)$ captures the probability of failure of the algorithm to satisfy $\epsilon$-DP. 
The global sensitivity $\Delta_\ell$ of a real-valued 
function $\ell: \mathcal{D} \to \mathbb{R}^k$ is defined as the maximum amount 
by which $\ell$ changes  in two adjacent inputs:
\(
  \Delta_\ell = \max_{D \sim D'} \| \ell(D) - \ell(D') \|.
\)
In particular, the Gaussian mechanism, defined by
\(
    \mathcal{M}(D) = \ell(D) + \mathcal{N}(0, \Delta_\ell^2 \, \sigma^2), 
\)
\noindent where $\mathcal{N}(0, \Delta_\ell^2\, \sigma^2)$ is 
the Gaussian distribution with $0$ mean and standard deviation 
$\Delta_\ell^2\, \sigma^2$, satisfies $(\epsilon, \delta)$-DP for 
$\delta \!>\! \frac{4}{5} \exp(-(\sigma\epsilon)^2 / 2)$ 
and $\epsilon \!<\! 1$ \cite{dwork:14}. 

\smallskip
\section{Problem settings and goals}
\label{sec:problem}

{The paper adopts boldface symbols to describe vectors (lowercase) and matrices (uppercase). Italic symbols are used to denote scalars (lowercase) and data features or random variables (uppercase). Notation $\|\cdot\|$ is used to denote the $L_2$ norm.} 
The paper considers datasets $D$ consisting of $n$ individuals' data points $(X_i, A_i, Y_i)$, with $i \!\in\! [n]$ drawn i.i.d.~from an unknown distribution. Therein, $X_i \!\in\! \mathcal{X}$ is a feature vector, $A_i \!\in\! \mathcal{A}$ is a protected group attribute, 
and $Y_i \!\in\! \mathcal{Y}$ is a label.
For example, consider the case of a classifier that needs to predict the risks associated with a lending decision. The training example features $X_i$ may describe the individual's demographics, education, credit score, and loan amount, the protected attribute $A_i$ may describe the individual gender or ethnicity, and $Y_i$ represents whether or not the individual will default on the loan.
The goal is to learn a classifier $f_\btheta : \mathcal{X} \to \mathcal{Y}$, where $\btheta$ is a vector of real-valued parameters, 
that 
guarantees the \emph{privacy} of each individual data $(X_i,A_i,Y_i)$ in $D$. 
The model quality is measured in terms of a 
nonnegative \emph{loss function} $\ell: \mathcal{Y} \times \mathcal{Y} \to \mathbb{R}_+$, and the problem is that of minimizing the empirical risk (ERM) function:
\begin{flalign}
\label{eq:erm}
    \min_\btheta \cL(\btheta; D) = \frac{1}{n} \sum_{i=1}^n 
    \ell(f_\btheta(X_i), Y_i) .
    \tag{L}
\end{flalign}
For a group $a \in \cA$, the paper uses $D_a$ to denote the subset of $D$ containing exclusively samples whose group attribute $A = a$. 
The paper focuses on learning classifiers that protect the disclosure of the individuals' data using the notion of differential privacy and it analyzes the fairness impact (as defined next) of privacy on different groups of individuals. Importantly, the paper assumes that the attribute $A$ is not part of the model input during inference. 
\setcounter{equation}{0}

\noindent\textbf{Fairness}
The fairness analysis focuses on the notion of \emph{excessive risk}, a widely adopted metric in private learning \cite{NIPS2017_f337d999,ijcai2017548}. 
It defines the difference between the private and non private risk functions: 
\begin{flalign}
\label{def:excessiver_risk}
  R(\btheta, D) = \EE_{\tilde{\btheta}}  
  \left[ \cL( \tilde{\btheta}; D ) \right]
       - \cL( \btheta^*; D),
\end{flalign}
where the expectation is defined over the randomness of the private mechanism and $\tilde{\btheta}$ denotes the private model parameters while $\btheta^* = \argmin_\btheta \cL(\btheta; D)$.
The paper uses shorthands $R(\btheta)$ and $R_a(\btheta)$ to denote, 
respectively, the population-level $R(\btheta, D)$ excessive risk and the group level $R(\btheta, D_a)$ excessive risk for group $a$. 
Fairness is measured with respect to the \emph{excessive risk gap}:
\begin{flalign} 
\label{def:risk_gap}
\xi_a = | R_a(\btheta) - R(\btheta) |.
\end{flalign}
(Pure) fairness is achieved when $\xi_a = 0$ for all groups $a \in \cA$ and, thus, a private and fair classifier aims at minimizing the maximum excessive risk gap among all groups. 
The paper assumes that the private mechanisms are non-trivial, i.e., they minimize the population-level excessive risk $R(\btheta)$.

All proofs are reported in the Appendix, Section \ref{sec:missing_proofs}.

\section{Warm up: output perturbation}
\label{sec:optput_pert}
\newcommand{\indep}{\perp \!\!\! \perp}

The paper starts with analyzing fairness under the DP setting induced by an output perturbation mechanism. 
In this setting the analysis restricts to twice differentiable and convex loss functions $\ell$. 
Output perturbation is a standard DP paradigm in which noise calibrated to the function sensitivity is added directly to the output of the computation. In the context of the \emph{regularized} ERM problem, adding noise drawn from a Gaussian distribution $\mathcal{N}(0, \Delta_\ell^2 \sigma^2)$ to the optimal model parameters $\btheta^*$ ensures $(\epsilon, \delta)$-differential privacy \cite{chaudhuri2011differentially}.  Therein, $\Delta_\ell = \nicefrac{2}{n \lambda}$  with regularization parameter $\lambda$. 
The following result sheds light on the unfairness induced by this  mechanism.
\begin{theorem}
\label{thm:output}
Let $\ell$ be a twice differentiable and convex loss function and consider the output perturbation mechanism described above.  
Then, the excessive risk gap for group $a \in \cA$ \mbox{is approximated by:}
\begin{flalign}
    \xi_a \approx \frac{1}{2}\Delta^2_\ell\sigma^2  \left| 
    \Tr(\bH^a_\ell) - \Tr(\bH_\ell) \right|,
    \label{eq:corr_excessive_hessian}
\end{flalign}
where 
\( \bH^a_\ell \!=\!\! \nabla^2_{\btheta^*} \sum_{(X,A,Y) \in D_a} 
   \!\!\ell( f_{\btheta^*}(X), Y) \)
 is the Hessian matrix of the loss function, at the optimal parameters 
 vector $\btheta^*$, computed using the group data $D_a$, 
\( \bH_\ell \) is the analogous Hessian computed using the population
data $D$, and $\Tr(\cdot)$ denotes the trace of a matrix. 
\end{theorem}

The approximation above follows form a second order Taylor expansion 
of the loss function, linearity of expectation, and the properties of 
Gaussian distributions. 
{It uses that fact that the excessive risk $R_a(\btheta)$ for a 
group $a$ can be approximated as $\nicefrac{1}{2}\Delta^2_\ell\sigma^2 \Tr(\bH^a_\ell)$.} 
The proof is reported in Appendix \ref{sec:missing_proofs}.

Theorem \ref{thm:output} sheds light on the relation between fairness and the difference in the local curvatures of the losses $\ell$ associated with a group and the population and provides a necessary condition to guarantee pure fairness. 
It suggests that output perturbation mechanisms may introduce unfairness when the local curvatures associated with the loss function of different groups differ substantially from one another. Additionally, the unfairness level is proportional to the amount of noise $\sigma$ or, equivalently, inversely proportional to the privacy parameter $\epsilon$, for a fixed $\delta$.
Finally, it also suggests that groups with larger Hessian traces $\Tr(\bH_a^\ell)$ will have larger excessive risk compared to groups with smaller Hessian traces. 
An additional analysis on the reasons behind why different groups may have large differences in their associated Hessian traces is provided in Section \ref{sec:noise}. 

\begin{wrapfigure}[12]{r}{210pt}
\vspace{-12pt}
\centering
\includegraphics[width=\linewidth]{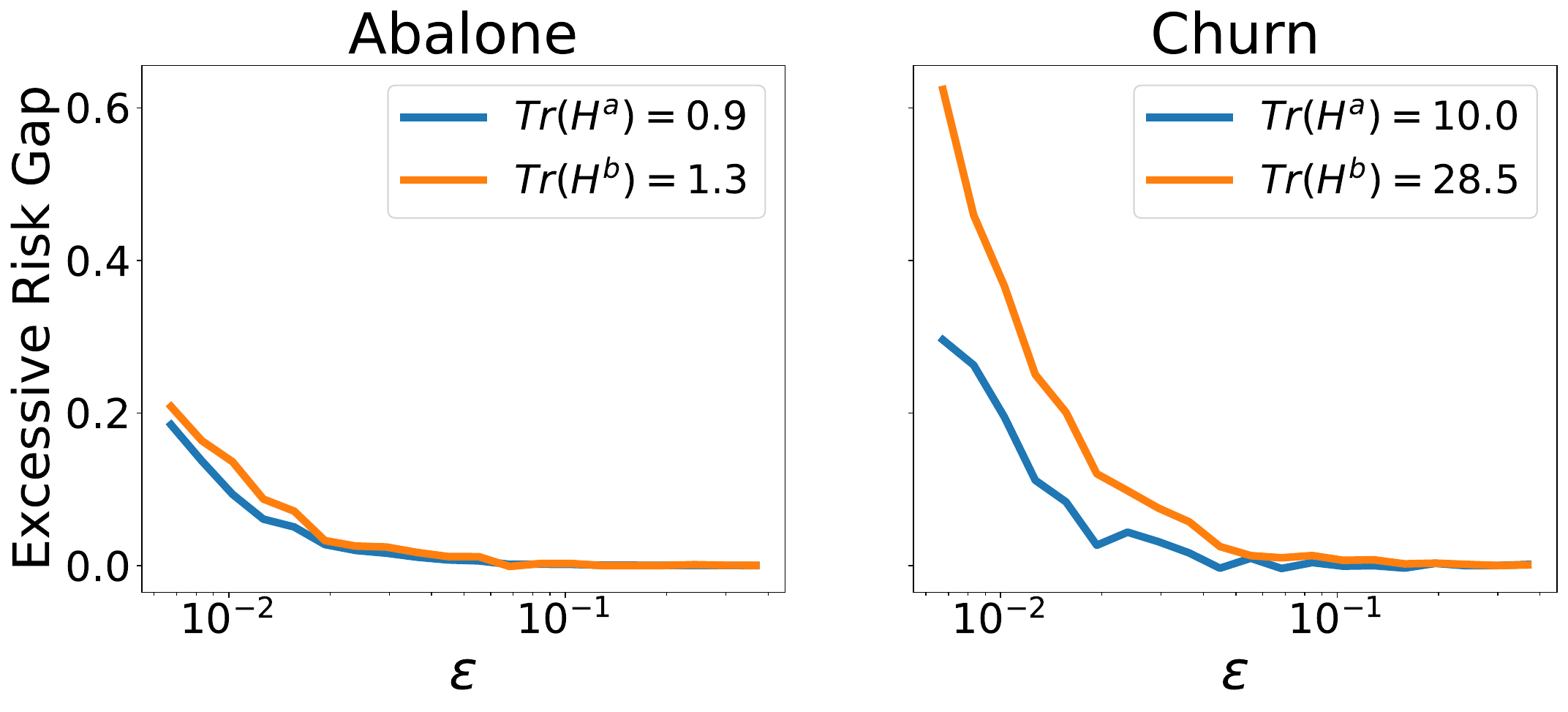}
\vspace{-12pt}
\caption{\small Correlation between excessive risk gap and Hessian 
Traces at varying of the privacy loss $\epsilon$.}
\label{fig:output_pertb}
\end{wrapfigure}

Figure \ref{fig:output_pertb} illustrates Theorem \ref{thm:output}. The plots show the correlation between the excessive risk gap $\xi_z(\btheta)$\footnote{In all experiment presented, the excessive risk gap is approximated by sampling over 100 repetitions.} and the quantity $\Tr(\bH_z^\ell)$ for each group $z \in \cA$, at varying of the privacy loss $\epsilon \in [0.005, 0.5], \delta=1e^{-5}$ on two datasets.
Each data point represents the average of 100 runs of a DP Logistic Regression (obtained with output perturbation) on each group $z \in \cA$. Details on dataset and experimental setting are provided in Appendix \ref{sec:experimental_settings} and additional experiments in Appendix \ref{sec:additional_experiments}.
Note the positive correlation between the excessive risk and the Hessian trace: \emph{Groups with larger Hessian trace tend to have larger excessive risks}. 
Note also the inverse correlation between $\epsilon$ and the dependency between the excessive risk and the Hessian trace. This is due to that larger $\epsilon$ values require smaller $\sigma$ values, and thus, as shown in Equation \ref{eq:corr_excessive_hessian}, the dependency between the excessive risk and Hessian trace is attenuated.

The following illustrates that even a class of simple linear models may not to satisfy pure fairness.
\begin{corollary}
\label{cor:l2_loss}
Consider the ERM problem for a linear model $f_{\btheta}(X) \defeq \btheta^T X$, with $L_2$ loss function i.e., $\ell(f_{\btheta}(X), Y) = (f_{\btheta}(X) - Y)^2$. Then, output perturbation does not guarantee pure fairness.
\end{corollary}

\noindent 
It follows from the observation that the Hessian of the $L_2$ loss for group $a \!\in\! \cA$, i.e.,
$\Tr(\bH^a_{\ell}) \!=\! \mathbb{E}_{X \sim D_a} \Tr(X X^T) \!=\!  
\mathbb{E}_{X \sim D_a} \|X\|^2$, depends solely on the input norms 
of the elements in $D_a$\footnote{Throughout the paper, we abuse notation and treat the dataset $D_Z$ associated with group $Z$ as distributions.}. 
Interestingly, this result highlights the relation between fairness and the average input norms of different group elements. When 
these norms are substantially different one another they will impact 
their respective excessive risks differently. {An additional analysis on this behavior is also discussed in Section \ref{sec:clipping}.} 

The following is a positive result.
\begin{corollary}
\label{cor:same_risk}
If for any two groups $a,b \in \cA$ their average group norms $\EE_{X_a \sim D_a} \| X_a \| = \EE_{X \sim D_b} \| X_b \| $ have identical values, then output perturbation with $L_2$ loss function provides pure fairness.
\end{corollary}
The above is a direct consequence of Corollary \ref{cor:l2_loss}. 
{\em Note also that pure fairness may be achieved, in this setting, by normalizing the input values for each group independently} (as shown in Appendix \ref{sec:additional_experiments}) although this solution requires accessing the sensitive group attributes at inference time.


\section{Gradient perturbation: DP-SGD}
\label{sec:dp-sdg}

Having identified the dependency between the Hessian of the model loss and the privacy parameters with the excessive risk gap 
in output perturbation mechanisms, this section extends the analysis to the context of DP Stochastic Gradient Descent (DP-SDG) \cite{abadi2016deep}. In contrast to output perturbation, DP-SGD does not restrict focus on convex loss functions and the privacy analysis does not require optimality of the model parameters $\btheta$, rendering it an appealing framework for DP ERM problems. 

\begin{wrapfigure}[12]{r}{220pt}
{\small
\begin{algorithm}[H]
  \caption{DP-SGD\!\!\!\!\!\!\!\!\!\!\!\!\!\!\!\!}
  \label{alg:alg1}
  \setcounter{AlgoLine}{0}
  \SetKwInOut{Input}{input}

  \Input{Training data $D$; Sample prob.~$q$; 
      Iterations $T$;~Noise\!\!\!\!\!\!\!\!\!\!\!\!\!\!\!\!\\variance $\sigma^2$; Clipping bound $C$; learning rate $\eta$}
  \label{line:1a}
  $\btheta_0 \gets \bm{0}^T$\\
  \label{line:2a}
  \For{iteration $t =  1,2, \ldots T$} {    
    \label{line:3a}
    $B \gets $ random sub-sample of $D$ with $\Pr q$\\
    \label{line:4a}
    \ForEach{ $(X_i, A_i, Y_i) \in B$} {
      \label{line:5a}
      $\bm{g}_i = \nabla \ell\left( f_{\btheta_t} (X_i), Y_i \right)$
    }
    \label{line:6a}
    $\bar{\bm{g}}_B \gets \frac{1}{|B|} 
    \hugP{\sum_i \pi_C(\bm{g}^i) + \cN(0, \bm{I} C^2 \sigma^2)}$\\
    $\btheta_{t+1} \gets \btheta_t - \eta \bar{\bm{g}}_B$
  }
\end{algorithm}
}
\end{wrapfigure}
In a nutshell, DP-SDG computes the gradients for each data sample in 
a random mini-batch $B$, clips their $L_2$-norm, 
adds noise to ensure privacy, and computes the average. 
Two key characteristics of DP-SGD are: {\bf (1)} Clipping the 
gradients whose $L_2$ norm exceeds a given bound $C$, and {\bf (2)} 
Perturbing the averaged clipped gradients with $0$-mean Gaussian 
noise with variance $\sigma^2 C^2$. 
The procedure is described in Algorithm \ref{alg:alg1}. Therein, 
$\bm{g}_i$ represents the gradient of a data sample $(X_i, A_i, Y_i)$,
$\bar{\bm{g}}_B$ the average clipped noisy gradient of the samples 
in mini-batch $B$, and the function $\pi_C (\bm{x}) = \bm{x} \cdot \min(1, \frac{C}{\| \bm{x} \|})$.

The following theorem is an important result of this section. 
It connects the expected loss $\EE[\cL(\btheta; D_a)]$ of a group $a \in \cA$ with its excessive risk $R_a(\btheta)$, which is, in turn, used in our fairness analysis. 
It decomposes the expected loss during private training into three key components: 
The first relates with the model parameters update and it is not 
affected by the private training. 
The other two relate with gradient clipping and noise addition, and, combined, capture the notion of excessive risk. 

\begin{theorem}
\label{thm:taylor}
Consider the ERM problem \eqref{eq:erm} with loss $\ell$ twice differentiable w.r.t.~the model parameters.
The expected loss $\EE[\cL(\btheta_{t+1}; D_a)]$ of group $a\!\in\!\cA$ at iteration $t\!+\!1$, is approximated as:
\begin{align}
\label{eq:thm2a}
 \mathbb{E}\left[\cL( \btheta_{t+1}; D_a )\right]
  & \approx  
  \underbrace{\cL(\btheta_{t}; D_a) - \eta 
  \hugA{\bm{g}_{D_a}, \bm{g}_D} 
  + \frac{\eta^2}{2} \mathbb{E}\hugB{\bm{g}_B^T \bH_{\ell}^a  \bm{g}_B}
  }_{ \mbox{non-private term}}
  \\
 \label{eq:thm2b}\tag{$R_a^{\clip}$}
 & + \underbrace{
    \eta \hugP{
    \hugA{ \bm{g}_{D_a}, \bm{g}_D } - 
    \hugA{ \bm{g}_{D_a}, \bar{\bm{g}}_D}
    } 
   + \frac{\eta^2}{2} 
    \hugP{
    \mathbb{E}\hugB{\bar{\bm{g}}_B^T \bH_{\ell}^a \bar{\bm{g}}_B}  - 
    \mathbb{E}\hugB{\bm{g}_B^T \bH_{\ell}^a \bm{g}_B}
   }
  }_{\mbox{private term due to clipping}} 
  \\
 \label{eq:thmc} \tag{$R_a^{\noise}$}
 & +\underbrace{
    \frac{\eta^2}{2} \Tr( \bH_{\ell}^a ) C^2 \sigma^2
    }_{\mbox{private term due to noise}}
\end{align}
where the expectation is taken over the randomness of the private 
noise and the mini-batch selection, and the terms $\bm{g}_Z$ and $\bar{\bm{g}}_Z$ denote, respectively, the average non-private and private gradients over subset 
$Z$ of $D$ at iteration $t$ (the iteration number is dropped for ease of notation).
\end{theorem}

The result in Theorem \ref{thm:taylor} follows from a second order Taylor expansion of the non-private and private ERM functions $\cL(\btheta_{t} -\eta \bm{g}_B; D_a)$ and $\cL(\btheta_{t} -\eta (\bm{\bar{g}}_B +\mathcal{\bm{N}}(0, \bm{I}C^2\sigma^2); D_a)$, respectively, around $\btheta_t$ and by comparing their differences. Once again, proofs are reported in Appendix \ref{sec:missing_proofs}.

The first term in the expression (Equation \eqref{eq:thm2a}) denotes 
the Taylor approximation of the (non-private) SGD loss. Terms 
\eqref{eq:thm2b} and \eqref{eq:thmc} quantify, together, the 
excessive risk for group $a$. Therein, \eqref{eq:thm2b} quantifies 
the effect of clipping to the excessive risk, and \eqref{eq:thmc} 
quantifies the effect of perturbing the average gradients to the 
excessive risk.
Therefore, Theorem \ref{thm:taylor} shows that there are two main 
sources of disparate impact in DP-SDG training:
\begin{enumerate}[leftmargin=*, parsep=0pt, itemsep=2pt, topsep=-4pt]
    \item {\em Gradient clipping} 
    ($R_a^\clip$): which, in turn, depends of three factors: 
    {\bf (i)} The values of the Hessian matrix $\bH_\ell^a$ of the 
    loss function associated with group $a$;
    {\bf (ii)} The gradients values $\bg_{D_a}$ associated with the samples of group $a$; and
    {\bf (iii)} The clipping bound $C$, which appears in $\bar{\bg}_B$ and $\bar{\bg}_D$.
    
    \item {\em Noise addition}
    ($R_a^\noise$): which, in turn, depends on two factors:
    {\bf (i)} The values of the (trace of the) Hessian matrix 
    $\bH_\ell^a$ of the loss function associated with group $a$; and
    {\bf (ii)} The privacy loss parameters $(\epsilon, \delta, \Delta_\ell)$ (which, in turn, are characterized by the noise variance $C^2\sigma^2$).
\end{enumerate}
A schematic representation of these factors is shown in Figure 
\ref{fig:unfairness}. Therein,  $X_{D_a}$ denotes the features values $X \in \cX$ of the subset $D_a$ of $D$. 
{\em Theorem \ref{thm:taylor} entails that unfairness occurs whenever different groups have different values for any of the gradient clipping and noise addition excessive risk terms.} 

\begin{figure}[tb]
\centering
\includegraphics[width=0.7\linewidth]{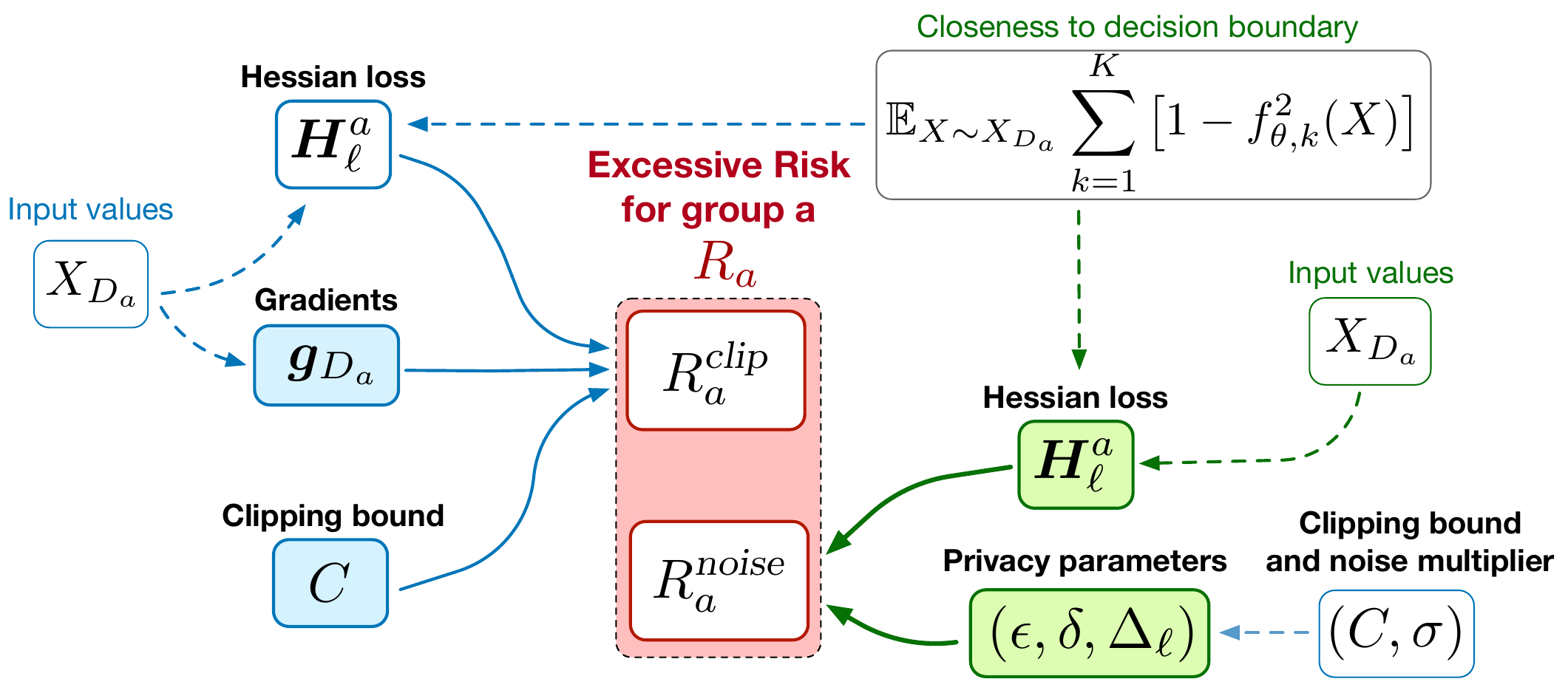}
\caption{\small Diagram of the factors affecting the excessive risk $R_a$ for a group $a \in \cA$ of individuals. 
Components affecting $R_a$ in output perturbation involve exclusively the green boxes while those affecting $R_a$ in DP-SGD involve both green and blue boxes. The main \emph{direct} factors (e.g., those appearing in Eq.~\eqref{eq:thm2a}) affecting the excessive risk clipping $R_a^\clip$ and noise $R_a^\noise$ components are highlighted within colored boxes. 
These direct factors are also regulated by \emph{latent} factors, shown in white boxes, with dotted lines illustrating dependencies.
\label{fig:unfairness}
}
\end{figure}

The next sections analyze the reasons behind the disparity in excessive risk focusing, independently, on terms $R_a^\clip$ (Section \ref{sec:clipping}) and $R_a^\noise$ (Section \ref{sec:noise}).
Independently studying these terms is motivated by observation that the clipping value $C$ regulates the dominance of a factor over the other. Indeed, for sufficiently large (small) $C$ values $R_a^\noise$ will dominate (be dominated by) $R_a^\clip$.\footnote{This observation relates with the bias-variance trade-off typically observed in DP-SGD \cite{tran2020differentially}.} 

\section{Why gradient clipping causes unfairness?}
\label{sec:clipping}

As highlighted above, there are three factors influencing the clipping effect to the excessive risk $R_a^\clip$: the \emph{Hessian loss}, the \emph{gradient values}, and the \emph{clipping bound}. 
This section illustrates their dependencies with the excessive risk, provides conditions to compare the disparate impacts between different groups, and shows the presence of an extra (latent) factor: the norm of the \emph{input values} $X_{D_a}$, which plays a role to this disparate impacts by indirectly controlling the norms of gradient $\bm{g}_{D_a}$ (see the diagram illustrated in Figure \ref{fig:unfairness}).

The next results assume that the empirical loss function $\cL(\btheta; D_a)$, associated with each group $a \in \cA$, is convex and $\beta_a$-smooth. The analysis also consider learning rates $\eta \leq \nicefrac{1}{\max_a \beta_a}$ and gradients $\bm{g}(B)$ and $\bm{\bar{g}}(B)$ with small variances. Note that this is not restrictive as the variance decreases as a function of the batch size $B$. 
Finally, for notational convenience, and w.l.o.g., the result focus on the case in which $|\cA| = 2$.
{As shown in the empirical assessment (see Appendix \ref{sec:additional_experiments}), however, the conclusions carry on even in cases when the above assumptions may not hold}.
\begin{wrapfigure}[18]{r}{140pt}
\vspace{-10pt}
\centering
    \includegraphics[width=1\linewidth]{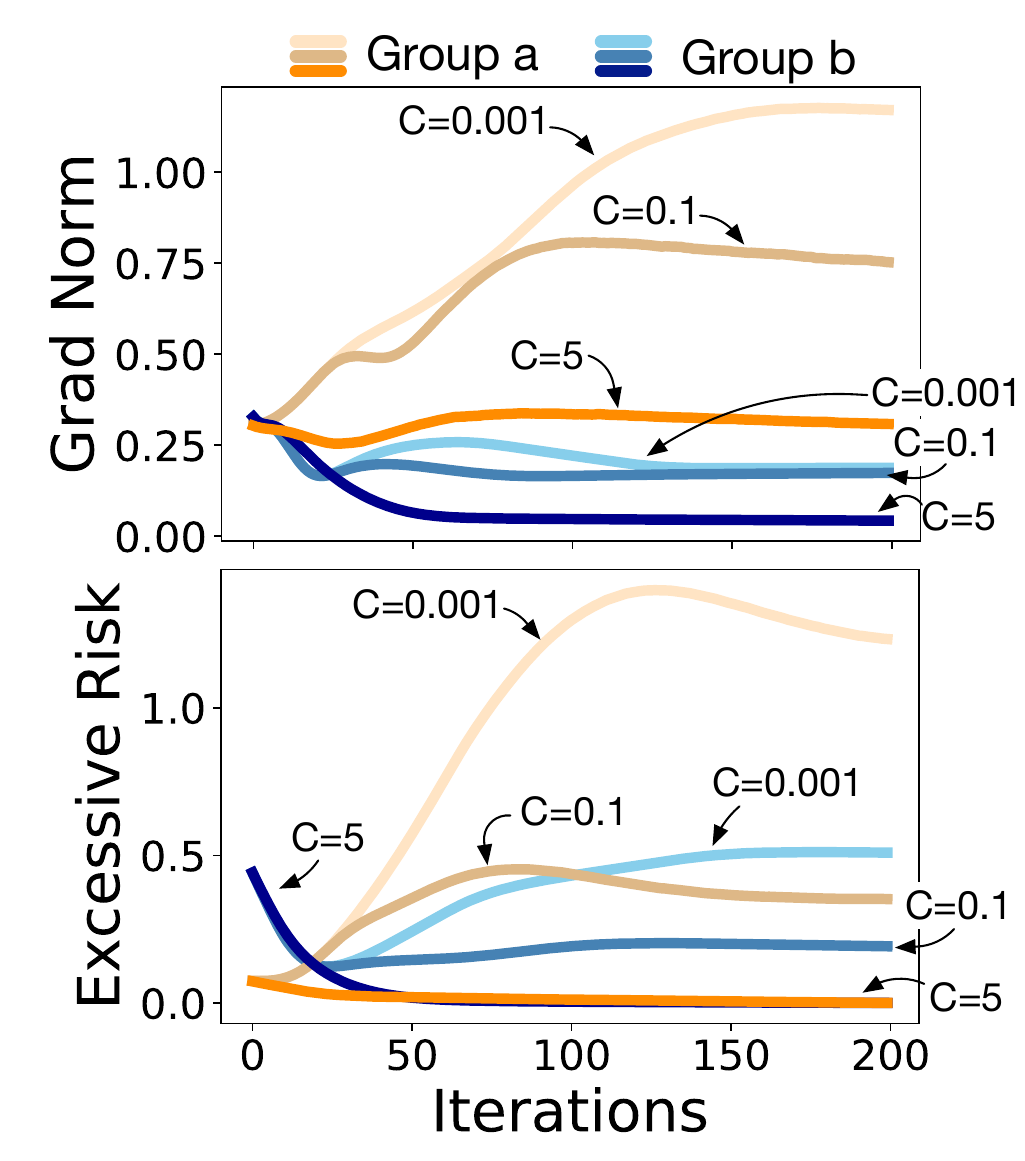}
    \vspace{-18pt}
    \caption{\small Impact of gradient clipping on gradient norms for 
    different clipping bounds. Bank dataset.
    \label{fig:rel_C_norm}}
\end{wrapfigure}

\begin{theorem}
\label{thm_large_grad_norm} 
Let $p_z = \nicefrac{|D_z|}{|D|}$ be the fraction of training samples in group $z \in \cA$. For groups $a, b \in \cA$, 
$R_a^\clip > R_b^\clip$ whenever:
\begin{equation}
\label{eq:grad_norm} 
  \norm{\bm{g}_{D_a}} \hugP{p_a - \frac{p_a
  ^2}{2}} \geq  
  \frac{5}{2} C +  \norm{\bm{g}_{D_b}}\hugP{1 + p_b +\frac{p_b^2}{2}}.
\end{equation}
\end{theorem}

Theorem \ref{thm_large_grad_norm} provides a sufficient condition for 
which a group may have larger excessive risk than another solely based 
on the clipping term analysis. 
{\em It relates unfairness with the average (non-private) gradient 
norms of the groups $\bm{g}_{D_a}$ and $\bm{g}_{D_b}$ and the 
clipping value $C$}. 
As shown in \mbox{the diagram of Figure~\ref{fig:unfairness}}, this result relates two main factors to the excessive risk due to clipping $\bm{R}^{clip}_a$: 
{\bf (1)} the \emph{clipping bound $C$}, and
{\bf (2)} the (norm of the) \emph{gradients $\|\bm{g}_{D_a}\|$}.
While the \emph{relative dataset size $p_a=\nicefrac{|D_a|}{|D|}$} of each group also appears in Equation \eqref{eq:grad_norm}, our extensive experiments showed that this factor may not play a prime role in controlling the disparate impacts (see Appendix \ref{sec:additional_experiments}).

The relation with these two factors is illustrated in Figure 
\ref{fig:rel_C_norm}, which shows the impact of gradient clipping (for different $C$ values) to the gradient norms (top) and to the excessive risk $R_a$ (bottom). It shows that the gradient norms reduce as $C$ increases and that the group with larger gradient norms have also larger excessive risk. 

Finally, the diagram in Figure \ref{fig:unfairness} also shows the presence of an 
additional factor affecting the gradient 
norms: \emph{the input norms}, whose average is denoted $X_{D_a} = 
\EE_{X \sim X_{D_a}} \|X\|$, in the figure.
While this aspect is not directly evident in Theorem 
\ref{thm_large_grad_norm}, the following examples highlight the 
positive correlation between input and gradients norms when considering a linear classifier and a feedforward neural network.  

\begin{example} 
\label{example:LR}
Consider the ERM problem \eqref{eq:erm} for a linear classifier $f_{\btheta}(X) \defeq \textsl{softmax}(\btheta^T X)$  and cross-entropy loss $\ell(f_{\btheta}(X),Y) = -\sum_{i=1}^K Y_i \log \boldsymbol{f}^i_{\btheta}(X)$ where $K$ is the number of classes. The gradient of the loss function at a given data point $(X,Y)$ is:
\(
    \bm{g}_{X} = \nabla{\btheta}\ell(f_{\btheta}(X),Y)   =  (\boldsymbol{Y} - \boldsymbol{f}) \otimes X.
\)
The result is by \cite{bohning1992multinomial} and it suggests that the gradient norms are proportional to the input norms: $\| \bm{g}_{X}\| \propto \|X\|$.
\end{example}

\begin{example} 
\label{example:NN}
Next, consider a neural network with single 
hidden layer, $f_{\btheta}(X) \defeq \textsl{softmax} \left(\btheta_1^T o(\btheta^T_2 X) \right)$, where $o(\cdot)$ is a proper 
activation function and $\btheta_1, \btheta_2$ are the model 
parameters. 
It can be seen that $\|g_X\| \propto \| \nabla_{\btheta_1} 
 \ell(f_{\btheta}(X),Y) \| + \| \nabla_{\btheta_2} \ell(f_{\btheta}(X),Y) \|$, 
 where $\| \nabla_{\btheta_2} \ell(f_{\btheta}(X),Y) \| \propto \| X\|$. 
{The full derivations are reported in Appendix \ref{sec:additional_examples}.}
\end{example}
Both examples illustrate a correlation between the gradients norms $\| \bm{g}_X\|$ and input norms $\|X\|$ for a given data sample $X$.
This behavior is also illustrated in Figure \ref{fig:corr_grad_input}, 
which highlights a positive correlation between the individual inputs 
and the gradients norms obtained while privately training a simple 
neural network (with one hidden layer) using DP-SGD on the Bank dataset. 
\begin{wrapfigure}[11]{r}{140pt}
\vspace{-6pt}
\centering
\includegraphics[width=\linewidth]{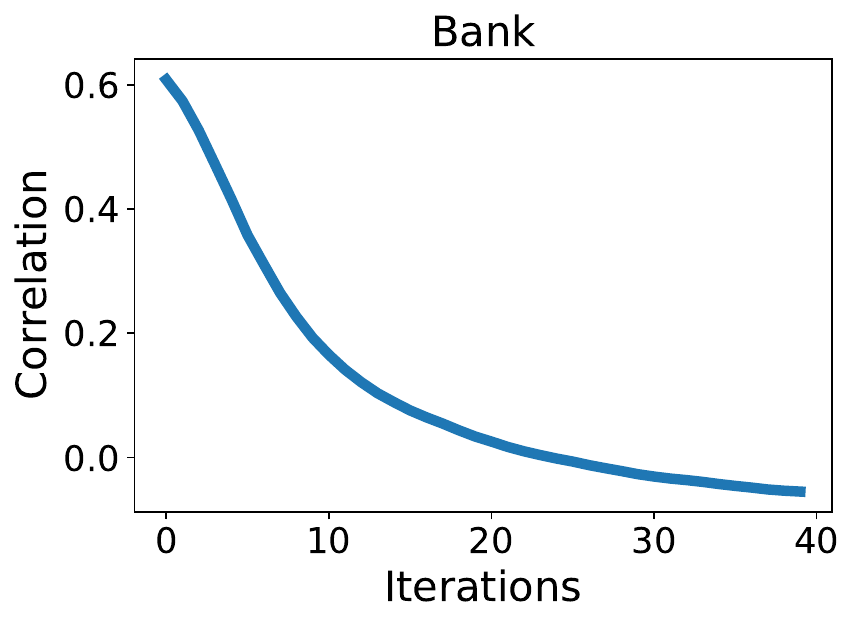}
\vspace{-18pt}
\caption{\small Correlation between inputs and gradients norms.}
\label{fig:corr_grad_input}
\vspace{-12pt}
\end{wrapfigure}
The experiment use $C=0.1$ and $\sigma =1$. The correlation 
decreases during training since the gradients norms reduce as training advances.

\emph{These observations imply that group data with large input 
norms---typically defining the tail of data distribution---result in 
large gradient norms and, thus, as shown in Theorem 
\ref{thm_large_grad_norm}, may have larger disproportionate impacts 
than groups with smaller input norms, under DP-SGD.} 
This analysis is in alignment with the empirical observation raised 
in \cite{NEURIPS2019_eugene}, showing that samples at the tail of a 
distribution may experience larger accuracy losses, in private 
training, with respect to~other samples.

While the above shows a dependency between gradients and clipping bound, as illustrated in the \eqref{eq:thm2b} equation,
the group excessive risk is also affected by the Hessian values. 
However, as shown in Appendix \ref{sec:additional_experiments}, the Hessian factor is almost always
dominated by the other factors examined in this section. 
This is due to the presence of the multiplier $\nicefrac{\eta^2}{2}$ which attenuate the impact of the Hessian value to the excessive risk due to clipping in conjunction with the smoothness assumptions, which prevents the Hessian values to grow too large. 

In summary, the main factors affecting $R^\clip_a$ for a group $a \in \cA$ are the norm of the group gradients $\bm{g}_{D_a}$, in turn 
controlled by the norm of the inputs $X_{D_a}$, and the clipping bound $C$.

\section{Why noise addition causes unfairness?}
\label{sec:noise}

Next, the paper analyzes the factors influencing the noise 
effect to the excessive risk $R_a^\noise$, which, as highlighted in 
Theorem \ref{thm:taylor}, for DP-SGD and Theorem  \ref{thm:output} for output perturbation, are the \emph{Hessian loss}, and the \emph{privacy loss parameters} $(\epsilon, \delta, \Delta_\ell)$ (see also Figure \ref{fig:unfairness}). 
Noting that the privacy parameters have a multiplicative effect on the Hessian loss (see Equations \eqref{eq:thmc} and \eqref{eq:corr_excessive_hessian}), the following analysis, treats them as constants, and restricts focus on the effects of the Hessian trace to the disparate impacts. 

The following result provide a condition to compare the disparate impacts between different groups, 
\begin{theorem}
\label{thm:add_noise}
For groups $a,b \in \cA$, $R^{noise}_a > R^{noise}_b$ whenever 
\[ \Tr(H^a_{\ell}) > \Tr(H^b_{\ell}).\]
\end{theorem}
Note the connection of the result above with Theorem \ref{thm:output}.
Additionally, as illustrated in the diagram of Figure \ref{fig:unfairness} the Hessian trace for a group is controlled by two (latent) factors: 
{\bf (1)} The average distance of the group data to the decision boundary, and 
{\bf (2)} The values of the group input norms. 
While these aspects are not directly evident in Theorem 
\ref{thm:add_noise}, the following highlights the 
positive correlation between these two factors and the Hessian Traces. 

\begin{example}
\label{exp:trace_hessian}
Consider the same setting of Example \ref{example:LR}.
The Hessian of the cross entropy loss of a sample $X \sim D$ is given by
$H^X_{\ell} = [\big(\mbox{diag}(\bm{f}) - \bm{f}\bm{f}^T  \big) \otimes X X^T ]$, where $\otimes$ is the Kronecker product \cite{bohning1992multinomial}. This result suggests that the trace of the Hessian for sample $X$ is proportional to its input norm: $\Tr(H^X_{\ell}) \propto  \|X\|^2$. Additionally it also shows that: $\Tr(H^X_{\ell}) \propto (1 - \sum_{k=1}^K \bm{f}_{\btheta, k}^2(X))$, where $K$ is the number of classes, whose term is connected to the distance to the decision boundary, as shown next.
\end{example}

The following result highlights the connection between the term 
$(1 - \sum_{k=1}^K \bm{f}_{\btheta, k}^2(X))$ and the distance of sample $X$ to the decision boundary.
\begin{theorem}
\label{thm:trace_hessian}
Consider a $K$-class classifier $\bm{f}_{\btheta,k}$ ($k \in [K]$). 
For a given sample $X \sim D$, the term
$\hugP{1 - \sum_{k=1}^K \bm{f}_{\btheta, k}^2(X)}$ 
is maximized when
$\bm{f}_{\btheta, k}(X) = \nicefrac{1}{K}$ and minimized when 
$\exists k \in [K]$ s.t.~$\bm{f}_{\btheta, k}(X) = 1$ 
and $\bm{f}_{\btheta, k'} = 0 \ \forall k'\in [K], k \neq k$.
\end{theorem}
That is, the term $\hugP{1 - \sum_{k=1}^K \bm{f}_{\btheta, k}^2(X)}$ 
is maximized (minimized) when the sample $X$ is close (far) to the decision boundary. Since, as shown in Example \ref{exp:trace_hessian} this term can be proportional to the Hessian trace, then the aforementioned relation also indicates a connection between the Hessian trace value for a sample and its distance to the decision boundary: 
The closest (farther) is a sample $X$ to the decision boundary the larger (smaller) is the associated Hessian trace value $\Tr(\bm{H}_\ell^X)$.
\begin{wrapfigure}[12]{r}{150pt}
\vspace{-8pt}
\centering
\includegraphics[width=\linewidth]{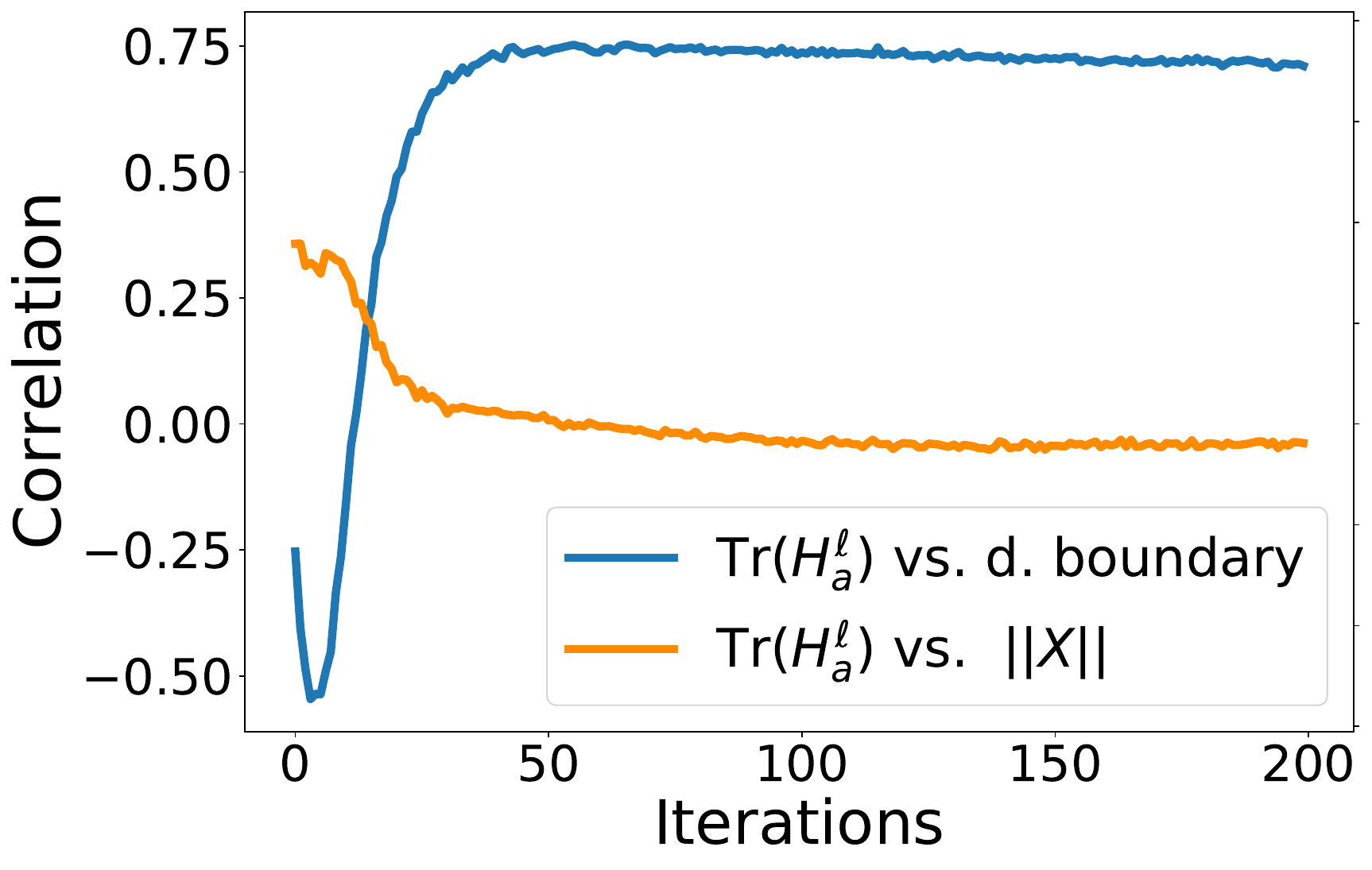}
\vspace{-14pt}
\caption{\small Correlation between trace of Hessian with closeness to 
boundary (dark color) and input norm (light color).}
\label{fig:corr_hessia_input}
\end{wrapfigure}
This is intuitive as the model decision are less robust to the presence of noise in the model (e.g., as that introduced by a DP mechanism) for the samples which are close to the decision boundary w.r.t.~those which are far from it. 

An analogous behavior is also observed in Neural Networks and described in {Appendix \ref{sec:additional_examples}} due to space constraints. 
Figure \ref{fig:corr_hessia_input} illustrates this behavior using the same setting adopted in Figure \ref{fig:corr_grad_input}. It highlights the positive correlation between the input norm, the trace of Hessian, and the closeness to the decision boundary for a given sample $X$.




While the above discusses the relation between input norms and Hessian losses, Figure \ref{fig_corr_input_excessive_risk} illustrates this dependencies with the excessive risk, which is one of the main objective of the analysis, on three datasets. \emph{Once again this observation recognizes the difference in input 
norms as a crucial proxy to unfairness: Groups with larger input 
norms will tend to have larger disproportionate impacts under private
training than groups with smaller input norms.} 

\begin{figure}[tb]
\centering
\includegraphics[width=0.75\linewidth]{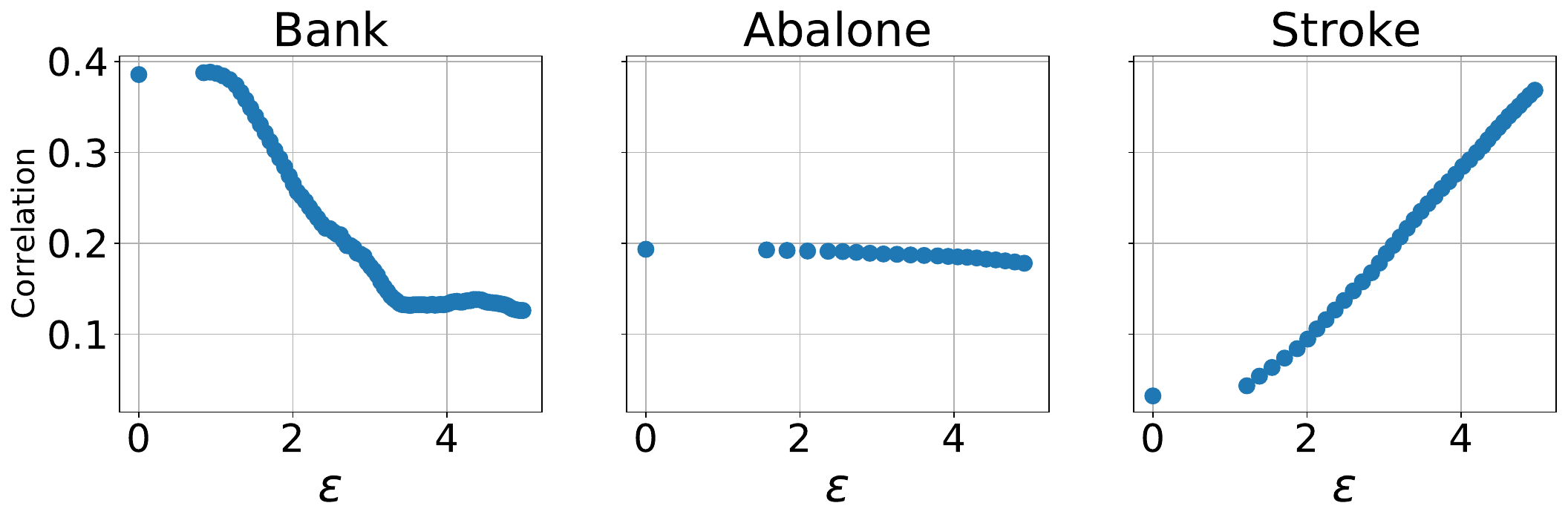}
\caption{Correlation between input norms and excessive risk; 
DP-SGD with $C=0.1$ and $\sigma=1.0$.}
\label{fig_corr_input_excessive_risk}
\end{figure}

In summary, the main factor affecting $R^\noise_a$ for a group $a \in \cA$ is the Hessian loss $\bm{H}_\ell^a$, which, in turn, is controlled by the group's distance to the decision boundary and by their inputs norm.

\section{Mitigation solution}
\label{sec:solution}
The previous sections showed that, in DP-SGD, the excessive risk 
$R_a$ for a group $a \in \cA$ could be decomposed into two factors 
$R_a^\clip$, due to clipping, and $R_a^\noise$, due to noise 
addition. 
In turn, it identified the gradients values $\bg_{D_a}$ associated 
with the samples of group $a$ and the clipping bound $C$ as the main 
sources of disparate impact in component $R_a^\clip$, 
and the (trace of the) Hessian $\bH_\ell^a$ of the group $a$ loss 
function as the main source of disparate impact in component $R_a^\noise$.

A solution to mitigate the effects of these components to 
the excessive risk gap is to equalize the factors responsible for
$R_a^\clip$ and $R_a^\noise$ among all group $a\in \cA$ during private 
training. The resulting empirical risk loss becomes:
\begin{align}
\label{eq:mitigation}
 \min_{\btheta} \cL(\btheta; D) +  \sum_{a \in \cA} 
 \hugP{
  \gamma_1 \abs{ \hugA{\bg_{D_a} - \bg_{D}, \bg_{D} - \bar{\bg}_{D}}} +
  \gamma_2 \abs{\Tr(\bH_\ell^a) - \Tr(\bH_\ell)}
  },
\end{align}
where the component multiplied by $\gamma_1$ comes for simplifying 
the expression 
$|\hugA{\bg_{D_a}, \bg_{D}} - \hugA{\bg_{D_a}, \bar{\bg}_{D}} -$
$\hugA{\bg_{D}, \bg_{D}} - \hugA{\bg_{D}, \bar{\bg}_{D}}|$ 
associated to the empirical risk gap $\xi_a$ of the main factor affecting
$R_a^\clip$, and component multiplied by $\gamma_2$ by the analogous 
expression for the main factor affecting $R_a^\noise$.
Note that this last component involves computing the Hessian matrices 
of the loss functions during each training step, which is a computationally expensive 
process. The previous section, however, showed a strong dependency
between the trace of the Hessian losses and the distance to the decision
boundary (Theorem \ref{thm:trace_hessian}). Thus, in place of Equation 
\eqref{eq:mitigation} the proposed mitigating solution solves:
\begin{align*}
\label{eq:mitigation2}
 \min_{\btheta} \cL(\btheta; D) +  \sum_{a \in \cA} 
 \hugP{
  \gamma_1 \abs{ \hugA{\bg_{D_a} - \bg_{D}, \bg_{D} - \bar{\bg}_{D}}} +
  \gamma_2 \abs{ \EE_{X \sim D_a} [1 \!-\! \sum_{k=1}^K f_{\theta,k}^2(X)] 
                  -
                  \EE_{X \sim D} [1 \!-\! \sum_{k=1}^K f_{\theta,k}^2(X)]
                }
  }.
\end{align*}
Figure \ref{fig:bank_mitigation} illustrates this approach at work, 
for various multipliers $\gamma_1$ and $\gamma_2$ on the Bank dataset 
with two protected group (blue = majority; orange = minority). 
Similar trends are shown for other datasets as well in Appendix 
\ref{sec:additional_experiments}.
The implementation uses a neural network with a single hidden layer and Suppose
uses DP-SDG with $C=0.1, \sigma=5.0$. A clear trend arises: For appropriately 
selected values $\gamma_1$ and $\gamma_2$ the excessive risk gap between
the majority and minority groups not only tends to be equalized, but 
it also decreases significantly for both groups. {\em These results imply 
that the proposed mitigating strategy may not only improve fairness but also the loss in utility of the private models.}
\begin{figure}[tb]
\centering
\includegraphics[width=0.9\linewidth]{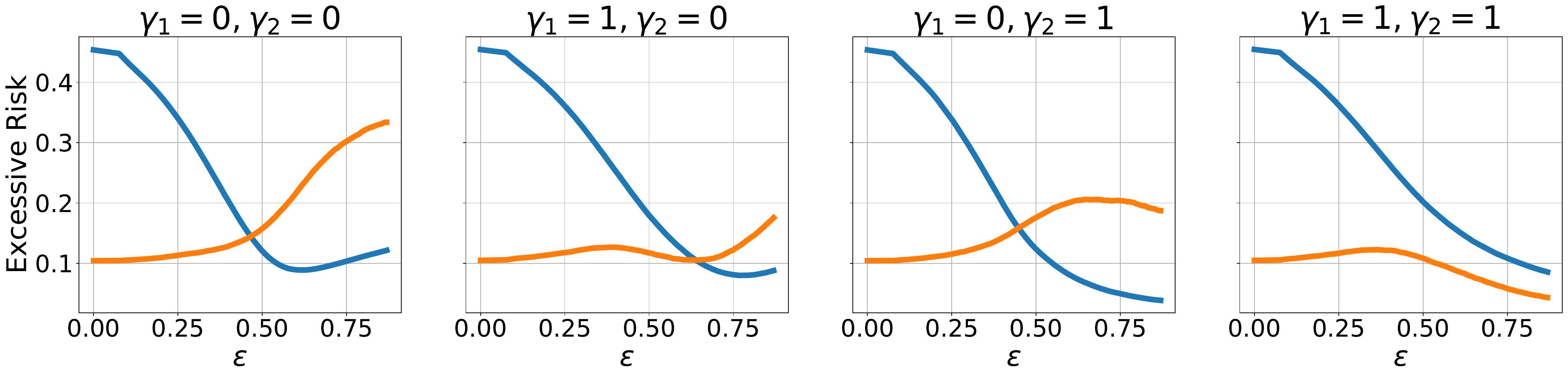}
\caption{Mitigating solution: Excessive risk gap at varying of the privacy
loss $\epsilon$ on the Bank dataset for different values of  
$\gamma_1$ and $\gamma_2$. Majority (minority) group is shown in dark (light)
colors.
\label{fig:bank_mitigation}}
\end{figure}


\section{Limitations and conclusions}
\label{sec:limitations}
This work was motivated by the recent observations regarding the disparate impacts induced by DP in learning systems. 
The paper introduced a notion of fairness that relies on the concept of excessive risk, analyzed this fairness notion in output perturbation and DP-SGD for ERM problems, it isolated the relevant components related with noise addition and gradient clipping 
responsible for the disparate impacts, studied the main factors affecting these components, and introduced a mitigation solution. 

This study recognizes the following limitations: Firstly, the analyses in Section \ref{sec:clipping} requires the ERM losses to 
be smooth and convex. While these are common assumptions adopted in the analysis of private ERM \cite{ijcai2017548, JMLR:v12:chaudhuri11a}, the generalization to the non-convex case is an interesting open question. 
The second limitation regards the selection of the multipliers 
$\gamma_1$ and $\gamma_2$ in Equation~\ref{eq:mitigation}. 
While the paper does not investigate how to optimally 
selecting these values, the adoption of a Lagrangian Dual framework, 
as in \citep{tran2020differentially}, could a useful tool to the automatic selection of such parameters, for an extra privacy 
cost. Finally, the proposed mitigation solution negatively affects the training runtime and the design of more efficient solutions and implementations is an interesting challenge.

Despite these limitations, given the increasingly key role of differential privacy in machine learning, we believe that this work may represent an important and broadly useful step toward understanding the roots of the disparate impacts observed in differentially private learning systems.

\section*{Acknowledgement}
This research is partially supported by an Amazon Research Award.

\bibliographystyle{abbrvnat}
\bibliography{lib}

\appendix
\pagenumbering{arabic}
\renewcommand{\thepage} {A--\arabic{page}}

\makeatletter
\@addtoreset{theorem}{section}
\@addtoreset{corollary}{section}

\makeatother
\section{Missing Proofs}
\label{sec:missing_proofs}

\begin{theorem}
\label{thm:output_pertb}
Let $\ell$ be a twice differentiable and convex loss function and consider the output perturbation mechanism described above.  
Then, the excessive risk gap for group $a \in \cA$ \mbox{is approximated by:}
\begin{equation}
\label{eq:thm3}
    \xi_a \approx \frac{1}{2} \Delta^2_\ell\sigma^2  \left| 
    \Tr(\bH^a_\ell) - \Tr(\bH_\ell) \right|, \tag{3}
\end{equation}
where 
\( \bH^a_\ell \!=\!\! \nabla^2_{\btheta^*} \sum_{(X,A,Y) \in D_a} 
   \!\!\ell( f_{\btheta^*}(X), Y) \)
 is the Hessian matrix of the loss function at the optimal parameters 
 vector $\btheta^*$, computed using the group data $D_{a}$, 
\( \bH^\ell \) is the analogous Hessian computed using the population
data $D$, and $\Tr(\cdot)$ denotes the trace of a matrix. 
\end{theorem}

\begin{proof}

Recall that the output perturbation mechanism adds Gaussian noise
directly to the non-private model parameters $\btheta^{*}$ to obtain
the private parameters $\tilde{\btheta}$. Denote $\psi \sim \mathcal{N}
(0,  \mathbf{I}\Delta^2_{\ell} \sigma^2)$ the random noise vector
with the same size as $\btheta^{*}$. 
Then $\tilde{\btheta} = \btheta^{*} + \psi$.  
Using a second order Taylor expansion around $\btheta^{*}$ the 
private risk function for group $a \in \cA$ is approximated as follows:
\begin{equation}
\label{eq:7}
        \cL(\tilde{\btheta}, D_a) =  \cL(\btheta^* + \psi, D_a) \approx 
        \cL(\btheta^*, D_a) + \psi^T \nabla_{\btheta^*} \cL(\btheta^*, D_a) + 
        \frac{1}{2} \psi^T \bm{H}_\ell^a \psi.
\end{equation}

Taking the expectation with respect to $\psi$ on both sides of the above 
equation results in:
\begin{subequations}
\begin{align}
    \label{eq:8a}
    \mathbb{E} \hugB{\cL(\tilde{\btheta}, D_a)} &\approx 
    \cL(\btheta^*, D_a) + \mathbb{E} \hugB{\psi^T \nabla_{\btheta^*} \cL(\btheta^*, D_a)} 
    + \frac{1}{2} \mathbb{E} \hugB{\psi^T \bm{H}_\ell^a \psi} \\
    &= \label{eq:8b} 
    \cL(\btheta^*, D_a)
    + \frac{1}{2} \mathbb{E} \hugB{\psi^T \bm{H}_\ell^a \psi} \\
    &= \label{eq:8c} 
    \cL(\btheta^*, D_a)
    + \frac{1}{2} \sum_{i,j}  \mathbb{E}\hugB{\psi_i (\bm{H}_\ell^a)_{ij} \psi_j}\\
    &= \label{eq:8d} 
    \cL(\btheta^*, D_a)
    + \frac{1}{2} \sum_{i} \mathbb{E}\hugB{\psi^2_i} (\bm{H}_\ell^a)_{ii}\\
    &= \label{eq:8e} 
    \cL(\btheta^*, D_a)
    + \frac{1}{2} \Delta^2_{\ell} \sigma^2 \Tr\hugP{\bm{H}_\ell^a},
\end{align}
\end{subequations}
where equation \eqref{eq:8b} follows from linearity of expectation, by observing that $\nabla_{\btheta^*} \cL(\btheta^*, D_a)$
is a constant term, and that $\psi$ is a $0$-mean noise variable, 
thus, $\mathbb{E}[\psi]=\bm{0}^T \times \nabla_{\btheta^*} \cL(\btheta^*, D_a) = \bm{0}^T$.
Equation \eqref{eq:8c} follows by definition of Hessian matrix, where 
$(H_\ell^a)_{ij}$ denotes the entry with indices $i$ and $j$ of the matrix.
Equation \eqref{eq:8d} follows from that $\psi_i \perp \psi_j$, 
for all $i \neq j$, and Equation \eqref{eq:8e} from that
for a random variable $X$, $\EE[X^2] = (\EE[X])^2 + \var[X]$, and 
$\var[\psi_i] = \Delta^2_{\ell} \sigma^2\,\, \forall i$ and definition of Trace of a matrix.

Therefore, the group and population excessive risks are approximated as: 
\begin{align}
    \label{eq:9}
  R_a(\btheta) &= \mathbb{E}\hugB{\cL(\tilde{\btheta}, D_a)} - \cL(\btheta^*, D_a) 
              \approx  \frac{1}{2}\Delta^2_\ell \sigma^2 \Tr\hugP{\bm{H}_\ell^a}\\
    \label{eq:10}
    R(\btheta) &= \mathbb{E}\hugB{\cL(\tilde{\btheta}, D)} - \cL(\btheta^*, D) 
               \approx \frac{1}{2}\Delta^2_\ell \sigma^2 \Tr \hugP{\bm{H}_\ell}.
\end{align}
The claim follows by definition of excessive risk gap (Equation \ref{def:risk_gap})
subtracting Equation \eqref{eq:9} from \eqref{eq:10} in absolute values.
\end{proof}

\begin{corollary}
\label{cor:l2_loss}
Consider the ERM problem for a linear model $f_{\btheta}(X) \defeq \btheta^T X$, with $L_2$ loss function i.e., $\ell(f_{\btheta}(X), Y) = (f_{\btheta}(X) - Y)^2$. Then, output perturbation does not guarantee pure fairness.
\end{corollary}

\begin{proof}
First, notice that for an $L_2$ loss function the trace of Hessian loss for a group $a \in \cA$ is: 
$$  \Tr(\bH^a_\ell)  = \mathbb{E}_{x \sim D_a} \|X\|. $$

Therefore, from Theorem \ref{thm:output_pertb}, the excessive risk gap $\xi_a$ for group $a$ is:
\begin{equation}
\label{eq:11}
    \xi_a \approx \frac{1}{2} \Delta^2_\ell\sigma^2  \left| 
     \mathbb{E}_{x \sim D_a} \|X\| -  \mathbb{E}_{x \sim D} \|X\|  \right|.
\end{equation}
Notice that $\xi_a$ is larger than zero only if the average input norm of group $a$ is different with that of the population one. 
Since this condition cannot be guaranteed in general, the output perturbation mechanism for a linear ERM model under the $L_2$ loss does not guarantee pure fairness.
\end{proof}

\begin{corollary}
\label{cor:same_risk}
If for any two groups $a,b \in \cA$ their average group norms $\EE_{X_a \sim D_a} \| X_a \| = \EE_{X_b \sim D_b} \| X_b \| $ have identical values, then output perturbation with $L_2$ loss function provides pure fairness.
\end{corollary}

\begin{proof}
The above follows directly by observing that, when the average norms 
of any two groups have identical values, $\xi_a \approx 0$ for any group
 $a \in \cA$ (see Equation \eqref{eq:11}), and thus the average norm of each group also coincide with that of the population. 
\end{proof}
The above indicates that as long as the average group norm is invariant across 
different groups, then output perturbation mechanism provides pure fairness. 

\begin{theorem}
\label{thm:taylor}
Consider the ERM problem \eqref{eq:erm} with loss $\ell$ twice differentiable with respect to the model parameters.
The expected loss $\EE[\cL(\btheta_{t+1}; D_a)]$ of group $a\!\in\!\cA$ at iteration $t\!+\!1$, is approximated as:
\begin{align}
\label{eq:thm2a}
 \mathbb{E}\left[\cL( \btheta_{t+1}; D_a )\right]
  & \approx  
  \underbrace{\cL(\btheta_{t}; D_a) - \eta 
  \hugA{\bm{g}_{D_a}, \bm{g}_D} 
  + \frac{\eta^2}{2} \mathbb{E}\hugB{\bm{g}_B^T \bH_{\ell}^a  \bm{g}_B}
  }_{ \mbox{non-private term}} \tag{4}
  \\
 \label{eq:thm2b}\tag{$R_a^{\clip}$}
 & + \underbrace{
    \eta \hugP{
    \hugA{ \bm{g}_{D_a}, \bm{g}_D } - 
    \hugA{ \bm{g}_{D_a}, \bar{\bm{g}}_D}
    } 
   + \frac{\eta^2}{2} 
    \hugP{
    \mathbb{E}\hugB{\bar{\bm{g}}_B^T \bH_{\ell}^a \bar{\bm{g}}_B}  - 
    \mathbb{E}\hugB{\bm{g}_B^T \bH_{\ell}^a \bm{g}_B}
   }
  }_{\mbox{private term due to clipping}} 
  \\
 \label{eq:thmc} \tag{$R_a^{\noise}$}
 & +\underbrace{
    \frac{\eta^2}{2} \Tr( \bH_{\ell}^a ) C^2 \sigma^2
    }_{\mbox{private term due to noise}}
\end{align}
where the expectation is taken over the randomness of the private 
noise and the mini-batch selection, and the terms $\bm{g}_Z$ and $\bar{\bm{g}}_Z$ denote, respectively, the average non-private and private gradients over subset 
$Z$ of $D$ at iteration $t$ (the iteration number is dropped for ease of notation).
\end{theorem}

\begin{proof}

The proof of Theorem \ref{thm:taylor} relies on the following two 
second order Taylor approximations: 
{\textbf{(1)}} The first approximates the ERM loss at iteration
$t+1$ under non-private training, i.e., $\btheta_{t+1} = \btheta_
{t} - \eta \bm{g}_B$, where $B \subseteq D$ denotes the minibatch.
{\textbf{(2)}} The second approximates expected ERM loss under
private-training, i.e $\btheta_{t+1} = \btheta_{t} - \eta (\bar{\bm{g}}_B 
+ \psi)$ where $\psi \sim \mathcal{N}(0, \bm{I} C^2\sigma^2)$. 
Finally, the result is obtained by taking the difference of 
these approximations under private and non-private training.

\paragraph{1. Non-private term.}
The non private term of Theorem \ref{thm:taylor} can be
derived using second order Taylor approximation as follows: 
\begin{equation}
\label{eq:12}
        \cL(\theta_{t+1}, D_a)= \cL(\theta_{t} - \eta \bm{g}_{B}, D_a) \approx \cL(\theta_{t}, D_a) -\eta \langle \bm{g}_{D_a},  \bm{g}_{B} \rangle + \frac{\eta^2}{2} \bm{g}_{B}^T \bm{H}_\ell^a \bm{g}_{B}
\end{equation}
Taking the expectation with respect to the randomness of the mini-batch $B$ selection 
on both sides of the above approximation, and noting that $\mathbb{E}[\bm{g}_{B}] =\bm{g}_D$ 
(as $B$ is selected randomly from dataset $D$), it follows: 
\begin{subequations}
\label{eq:13}
\begin{align}
\label{eq:13a}
        \mathbb{E} [\cL(\theta_{t+1}, D_a)]  
        &\approx \cL(\theta_{t}, D_a) -\eta \mathbb{E} [\langle \bm{g}_{D_a}, \bm{g}_{B}] \rangle  + \frac{\eta^2}{2}  \mathbb{E} [ \bm{g}_{B}^T \bm{H}_\ell^a \bm{g}_{B}] \\
\label{eq:13b}
        & = \cL(\theta_{t}, D_a)  - \eta \langle \bm{g}_{D_a}, \bm{g}_D \rangle + \frac{\eta^2}{2}  \mathbb{E} [ \bm{g}_{B}^T \bm{H}_\ell^a \bm{g}_{B}].
\end{align}
\end{subequations}


\paragraph{2. Private term (due to both clipping and noise).}
Consider the private update in DP-SGD, i.e., 
$\btheta_{t+1} = \btheta_{t} - \eta (\bar{\bm{g}}_B + \psi)$. 
Again, applying a second order Taylor approximation around $\btheta_{t}$ 
allows us to estimate the expected private loss at iteration $t+1$ as:
\begin{subequations}
\label{eq:taylor_dpsgd}
\begin{align}
        \cL(\btheta_{t+1}, D_a)
        & = \cL \hugP{\btheta_{t} - \eta (\bar{\bm{g}}_{B} +\psi), D_a} \notag\\
\label{eq:14a}
        & \approx \cL \hugP{\btheta_{t}, D_a} 
        - \eta \left\langle \bm{g}_{D_a},  \bar{\bm{g}}_{B}+\psi \right\rangle 
        + \frac{\eta^2}{2} \hugP{\bar{\bm{g}}_{B}  +\psi}^T \bm{H}_\ell^a \hugP{\bar{\bm{g}}_{B}  +\psi}\\
\label{eq:14b}
        &= \cL \hugP{\btheta_{t}, D_a} 
        - \eta  \left\langle  \bm{g}_{D_a},  \bar{\bm{g}}_{B}  \right\rangle  
        - \eta  \left\langle  \bm{g}_{D_a}, \psi  \right\rangle   
        +  \frac{\eta^2}{2} \bar{\bm{g}}_{B}^T \bm{H}_\ell^a \bar{\bm{g}}_{B} \\
\notag
        & \hspace{50pt} +  \frac{\eta^2}{2} \left( \psi^T \bm{H}_\ell^a \bar{\bm{g}}_{B} + \bar{\bm{g}}_{B}^T \bm{H}_\ell^a\psi +\psi^T \bm{H}_\ell^a \psi \right) 
\end{align}
\end{subequations}
Taking the expectation with respect to the randomness of the mini-batch $B$ selection
and with respect to the randomness of noise $\psi$ on both sides of the above equation
gives:
\begin{subequations}
\label{eq:15}
\begin{align}
\label{eq:15a}
    \EE\hugB{\cL(\btheta_{t+1}, D_a)} & \approx 
        \EE \Big[\cL \hugP{\btheta_{t}, D_a} 
        - \eta  \left\langle  \bm{g}_{D_a},  \bar{\bm{g}}_{B}  \right\rangle  
        - \eta  \left\langle  \bm{g}_{D_a}, \psi  \right\rangle   
        +  \frac{\eta^2}{2} \bar{\bm{g}}_{B}^T \bm{H}_\ell^a \bar{\bm{g}}_{B}\\
\notag
        &\hspace{100pt} 
        + \frac{\eta^2}{2} \left( \psi^T \bm{H}_\ell^a \bar{\bm{g}}_{B} + \bar{\bm{g}}_{B}^T \bm{H}_\ell^a\psi +\psi^T \bm{H}_\ell^a \psi \right)
        \Big]\\
\label{eq:15b}
        &= \cL(\btheta_{t}, D_a) 
        - \eta \left\langle\ \bm{g}_{D_a}, \bar{\bm{g}}_{B} \right\rangle\ 
        - \eta  \left\langle  \bm{g}_{D_a}, \EE[\psi]  \right\rangle   
        + \frac{\eta^2}{2} \mathbb{E}\hugB{ \bar{\bm{g}}_{B}^T \bm{H}_\ell^a  \bar{\bm{g}}_{B} }\\
        \notag
        & \hspace{100pt}
        + \frac{\eta^2}{2} \left( \EE\hugB{\psi}^T\!\! \bm{H}_\ell^a \bar{\bm{g}}_{B} 
                            + \bar{\bm{g}}_{B}^T \bm{H}_\ell^a \EE[\psi] 
                            + \EE\hugB{\psi^T \bm{H}_\ell^a \psi} \right)\\
\label{eq:15c}
        &= \cL(\btheta_{t}, D_a) 
        - \eta \left\langle\ \bm{g}_{D_a}, \bar{\bm{g}}_{B} \right\rangle\ 
        + \frac{\eta^2}{2} \mathbb{E}\hugB{ \bar{\bm{g}}_{B}^T\! \bm{H}_\ell^a  \bar{\bm{g}}_{B} }
        + \frac{\eta^2}{2} \EE\hugB{\psi^T \bm{H}_\ell^a \psi}\\
\label{eq:15d}
        &= \cL(\btheta_{t}, D_a) 
        - \eta \left\langle\ \bm{g}_{D_a}, \bar{\bm{g}}_{B} \right\rangle\ 
        + \frac{\eta^2}{2} \mathbb{E}\hugB{ \bar{\bm{g}}_{B}^T\! \bm{H}_\ell^a  \bar{\bm{g}}_{B} }
        + \frac{\eta^2}{2} \Tr\hugP{\bm{H}_\ell^a} C^2 \sigma^2,
\end{align}
\end{subequations}
where \eqref{eq:15b}, and \eqref{eq:15c} follow from linearity of expectation and 
from that $\mathbb{E}[\psi] = 0$, since $\psi$ is a $0$-mean noise variable. 
Equation \eqref{eq:15d} follows from that,   
\[
  \mathbb{E} \hugB{\psi^T \bm{H}_\ell^a \psi} = 
  \mathbb{E} \hugB{\sum_{i,j} \psi_i   (\bm{H}_\ell^a)_{i,j} \psi_j} = 
  \sum_{i} \mathbb{E} \hugB{\psi_i^2 (\bm{H}_\ell^a)_{i,i}} =  \Tr \hugP{\bm{H}_\ell^a} C^2\sigma^2,
\]
since $\EE[\psi^2] = \EE[\psi]^2 + \var[\psi]$ and $\EE[\psi] = 0$ while 
$\var[\psi] = C^2\sigma^2$. 

Note that in the above approximation (Equation \eqref{eq:15}), the component 
\begin{equation}
\label{eq:16}
    \cL(\btheta_{t}, D_a) 
    - \eta \langle \bm{g}_{D_a}, \bar{\bm{g}}_{B} \rangle 
    + \frac{\eta^2}{2} \mathbb{E}\hugB{\bar{\bm{g}}_{B}^T \bm{H}_\ell^a  \bar{\bm{g}}_{B}}
\end{equation}
is associated to the SGD update step in which gradients have been clipped to the 
clipping bound value $C$, i.e.~$\btheta_{t+1} = \btheta_{t} - \eta (\bar{\bm{g}}_B)$. 

Next, the component
\begin{equation}
\label{eq:17}
    \frac{\eta^2}{2} \Tr\left(\bm{H}_\ell^a\right) C^2 \sigma^2
\end{equation}
is associated to the SGD update step in which the noise $\psi$ is added to 
the gradients.

If we take the difference between the approximation associated 
with the non-private loss term, obtained in Equation \ref{eq:13b}, 
with that associated with the private loss term, obtained in Equation \ref{eq:15d}, 
we can derive the effect of a single step of (private) DP-SGD 
compared to its non-private counterpart: 
\begin{subequations}
\label{eq:18}
\begin{align}
\label{eq:18a}
 \mathbb{E}\left[\cL( \btheta_{t+1}; D_a )\right]
  & \approx   \cL(\btheta_{t}; D_a) 
    - \eta \hugA{\bm{g}_{D_a}, \bm{g}_D} 
    + \frac{\eta^2}{2} \mathbb{E}\hugB{\bm{g}_B^T \bH_{\ell}^a  \bm{g}_B} \\
\label{eq:18b}
    &\hspace{10pt} 
    + \eta \hugP{ \hugA{ \bm{g}_{D_a}, \bm{g}_D } - \hugA{ \bm{g}_{D_a}, \bar{\bm{g}}_D} } 
   + \frac{\eta^2}{2} 
    \hugP{
    \mathbb{E}\hugB{\bar{\bm{g}}_B^T \bH_{\ell}^a \bar{\bm{g}}_B}  -  \mathbb{E}\hugB{\bm{g}_B^T \bH_{\ell}^a \bm{g}_B}
   } \\
\label{eq:18c} 
  &\hspace{10pt}
   + \frac{\eta^2}{2} \Tr\hugP{\bH_{\ell}^a} C^2 \sigma^2.
\end{align}
\end{subequations}
In the above,
\begin{itemize}[leftmargin=*, parsep=0pt, itemsep=2pt, topsep=-4pt]
    \item The components in Equation \eqref{eq:18a} are associated with 
    the loss under non-private training (see again Equation \ref{eq:13b});
  \item The components in Equation \eqref{eq:18b} is associated with for excessive 
  risk due to gradient clipping;
   \item Finally, the components in Equation \eqref{eq:18c} is associated with 
    the excessive risk due to noise addition.
\end{itemize}
\end{proof}

Next, the paper proves Theorem \ref{thm:grad_clip}. This result 
is based on the following assumptions. 

\begin{assumption}
\label{asm:convex}
[Convexity and Smoothness assumption] 
For a group $a \in \cA$, its empirical loss function 
$\cL( \btheta, D_a )$ 
is convex and $\beta_a$-smooth.
\end{assumption}

\begin{assumption}
\label{asm:large_batchsize} 
Let $B \subseteq D$ be a subset of the dataset $D$, and consider a 
constant $\varepsilon \geq 0$. 
Then, the variance associated with the gradient norms of a random mini-batch
$B$, $\sigma^2_{B} = \var\hugB{ \| \bm{g}_{B} \|} \leq \varepsilon$ 
as well as that associated with its clipped counterpart, 
$\bar{\sigma}^2_{B} = \var\hugB{\|\bm{\bar{g}}_{B}\|} \leq \varepsilon$.
\end{assumption}
The assumption above can be satisfied when the mini-batch size is large 
enough. For example, the variance is $0$ when $|B| = |D|$.

\begin{assumption}
\label{asm:lr}
The learning rate used in DP-SDG $\eta$ is upper bounded by quantity
 $\nicefrac{1}{\max_{z \in \cA} \beta_{z}}$.
\end{assumption}

\begin{theorem}
\label{thm:grad_clip}
Let $p_z = \nicefrac{|D_z|}{|D|}$ be the fraction of training samples in group $z \in \cA$. For groups $a, b \in \cA$, 
$R_a^\clip > R_b^\clip$ whenever:
\begin{equation}
\label{eq:5}
  \norm{\bm{g}_{D_a}} \hugP{p_a - \frac{p_a
  ^2}{2}} \geq  
  \frac{5}{2} C +  \norm{\bm{g}_{D_b}}\hugP{1 + p_b +\frac{p_b^2}{2}}.
  \tag{5}
\end{equation}
\end{theorem}
To ease notation, the statement of the theorem above uses $\varepsilon = 0$ (See 
Assumption \ref{asm:large_batchsize}) but the theorem can be generalized to 
any $\varepsilon \geq 0$.

The following Lemmas are introduced to aid the proof of Theorem \ref{thm:grad_clip}.
\begin{lemma}
\label{lem:norm_bound}
Consider the ERM problem \eqref{eq:erm} solved with DP-SGD with clipping 
value $C$. The following average clipped per-sample gradients 
$\bar{\bm{g}}_Z$, where $Z \subseteq D$, has norm at most $C$.
\end{lemma}
\begin{proof}
The result follows by triangle inequality: 
\begin{align*}
 \left\| \bar{\bm{g}}_{D_Z} \right\| &= \left\| \frac{1}{|D_Z|} \sum_{i \in D_Z} \bar{\bm{g}}_{i} \right\| \\
        &\leq \frac{1}{|D_Z|}\sum_{i \in D_Z} \left\| \bar{\bm{g}}_{i} \right\| \\
        &= \frac{1}{|D_Z|}\sum_{i \in D_Z} \left\| \bm{g}_{i} \min \big(1, \frac{C}{\|\bm{g}_{i}\|} \big)  \right\| \\
        & \leq \frac{1}{|D_Z|} \sum_{i \in D_Z} C = C.
\end{align*} 
\end{proof}

The next Lemma 
derives a lower and an upper bound for the component 
$\mathbb{E}[\bar{\bm{g}}_{B}^T \bm{H}_\ell^a \bar{\bm{g}}_{B}]  
- \mathbb{E}[\bm{g}_{B}^T \bm{H}_\ell^a \bm{g}_{B}]$, which appears in
the excessive risk term due to clipping $R^{clip}_a$ for some group $a \in \cA$.

\begin{lemma}
\label{bound_4}
Consider the ERM problem \eqref{eq:erm} with loss $\ell$, solved with 
DP-SGD with clipping value $C$. Further, let $\varepsilon = 0$ 
(see Assumption \ref{asm:large_batchsize}). For any group $a \in \cA$, the following
inequality holds:
\begin{align}
  - \beta_a  \|\bm{g}_D\|^2   \leq  \mathbb{E}[\bar{\bm{g}}_B^T \bm{H}_\ell^a \bar{\bm{g}}_{B}]  - \mathbb{E}[\bm{g}_{B}^T \bm{H}_\ell^a \bm{g}_{B}] \leq \beta_a C^2  
\end{align}
\end{lemma}

\begin{proof}
Consider a group $a \in \cA$. 
By the convexity assumption of the loss function, the Hessian $\bm{H}_\ell^a$ is 
a  positive  semi-definite matrix, i.e.,~for all real vectors of appropriate 
dimensions $\bm{v}$, it follows that $\bm{v}^T \bm{H}_\ell^a \bm{v} \geq 0$. 

Therefore, for a subset $B \subseteq D$ the following inequalities hold:
\begin{itemize}[leftmargin=*, parsep=0pt, itemsep=2pt, topsep=-4pt]
    \item  $\bar{\bm{g}}_{B} \bm{H}_\ell^a \bar{\bm{g}}_{B} \geq 0$,
    \item  $\bm{g}_{B}^T \bm{H}_\ell^a \bm{g}_{B} \geq 0$ . 
\end{itemize}
Additionally their expectations $\mathbb{E}[\bar{\bm{g}}_{B} \bm{H}_\ell^a \bar{\bm{g}}_{B}]$ 
 and  $\mathbb{E}[\bm{g}_{B}^T \bm{H}_\ell^a \bm{g}_{B}]$  are non-negative.

By the smoothness property of the loss function, 
$\bar{\bm{g}}_{B}^T \bm{H}_\ell^a \bar{\bm{g}}_{B} \leq \beta_a \| \bar{\bm{g}}_{B}\|^2$, thus:
\begin{subequations}
\begin{align}
\label{eq:21a}
    \mathbb{E}[\bar{\bm{g}}_{B}^T \bm{H}_\ell^a \bar{\bm{g}}_{B}]
        &\leq \beta_a \mathbb{E} \left[ \| \bar{\bm{g}}_{B}\|^2 \right] \\
\label{eq:21b}
        &= \beta_a \big(  \mathbb{E}[ \| \bm{\bar{g}}_{B} \|]^2 + \var\hugB{\| \bar{\bm{g}}_{B } \|} \big)\\
 \label{eq:21c}
        &\leq \beta_a \big( C^2 + \bar{\sigma}^2_B \big) \\
\label{eq:21d}
        &\leq \beta_a (C^2 + \varepsilon),
\end{align}
\end{subequations}
where Equation \eqref{eq:21b} follows from that $\EE[X^2] = (\EE[X])^2 + \var[X]$,
Equation \eqref{eq:21c} is due to Lemma \ref{lem:norm_bound}, and 
finally, the last inequality is due to Assumption \ref{asm:large_batchsize}.

Therefore, since $\varepsilon = 0$ by assumption of the Lemma, 
the following upper bound holds:
\begin{align}
  &\mathbb{E}[\bar{\bm{g}}_{B}^T \bm{H}_\ell^a \bar{\bm{g}}_{B}]  - \mathbb{E}[ \bm{g}_{B}^T \bm{H}_\ell^a \bm{g}_{B}] \leq \beta_a C^2.   \label{ineq_3}
\end{align}

Next, notice that
\begin{subequations}
\begin{align}
\label{eq:22a}
    \mathbb{E}[\bar{\bm{g}}_B^T \bm{H}_\ell^a \bar{\bm{g}}_{B}]  - \mathbb{E}[\bm{g}_{B}^T \bm{H}_\ell^a \bm{g}_{B}] 
        &\geq 
        - \mathbb{E}[\bm{g}_{B}^T \bm{H}_\ell^a \bm{g}_{B}] \\ 
\label{eq:22b}
        &\geq 
        - \EE[ \beta_a \| \bm{g}_B \|^2]\\
\label{eq:22c}
        &= -\beta_a \hugP{ \EE[ \|\bm{g}_B\| ]^2 + \var\hugB{ \|\bm{g}_B\| } }\\
\label{eq:22d}
        &= -\beta_a \| \bm{g}_D \|^2,
\end{align}
\end{subequations}
where the inequality in Equation \eqref{eq:22a} follows since both terms 
on the left hand side of the Equation are non negative. 
Equation \eqref{eq:22b} follows by smoothness assumption of the loss
function. 
Equation \eqref{eq:22c} follows by definition of expectation of a random 
variable, since $\EE[X]^2 = \EE[X^2] + \var[X]$.
Finally, Equation \eqref{eq:22d} follows from that $\var[\bm{g}_B] \leq \varepsilon = 0$ 
by Assumption \ref{asm:large_batchsize}, and that $\varepsilon = 0$ 
by assumption of the Lemma, and thus the norms $\|g_B\| = \|g_D\|$ 
and, thus, $\EE[\bm{g}_B] = \bm{g}_D$.
Therefore if follows:
\begin{align}
 -\beta_a \|\bm{g}_D\|^2 \leq  & \mathbb{E}[\bar{\bm{g}}_{B}^T \bm{H}_\ell^a \bar{\bm{g}}_{B}]  - \mathbb{E}[\bm{g}_{B}^T \bm{H}_\ell^a \bm{g}_{B}].
\end{align}
which concludes the proof.
\end{proof}
Again, the above uses $\varepsilon = 0$ to simplify notation, but the 
results generalize to the case when $\varepsilon > 0$. In such a case,
the bounds require slight modifications to involve the term $\varepsilon$.

\begin{lemma}
\label{lemma3}
Let $a, b \in \cA$ be two groups. 
Consider the ERM problem \eqref{eq:erm} solved with DP-SGD with clipping value $C$
and learning rate $\eta \leq \nicefrac{1}{\max_{a \in \cA} \beta_a}$.
Then, the difference on the excessive risk due to clipping 
$R^a_{clip}  - R^b_{clip}$ is lower bounded as:
\begin{equation}
\label{eq:23}
    R^a_{clip}  - R^b_{clip} \geq \eta 
        \hugP{\langle \bm{g}_{D_a} -\bm{g}_{D_b},\bm{g}_D-\bar{\bm{g}}_D \rangle -\frac{1}{2} ( \|\bm{g}_D\|^2 + C^2 )}.
\end{equation}
\end{lemma}

\begin{proof}
Recall that $B \subseteq D$ is the mini-batch during the resolution of DP-SGD. 
Using the lower and upper bounds obtained from Lemma \ref{bound_4}, it follows:
\begin{subequations}
\begin{align}
\label{eq:24a}
        R^a_{clip}  - R^b_{clip}  &= \eta \hugP{
    \hugA{ \bm{g}_{D_a}, \bm{g}_D } - 
    \hugA{ \bm{g}_{D_a}, \bar{\bm{g}}_D}
    } 
   + \frac{\eta^2}{2} 
    \hugP{
    \mathbb{E}\hugB{\bar{\bm{g}}_B^T \bH_{\ell}^a \bar{\bm{g}}_B}  - 
    \mathbb{E}\hugB{\bm{g}_B^T \bH_{\ell}^a \bm{g}_B}
   } \\
\notag
   & \hspace{10pt} - \eta \hugP{
    \hugA{ \bm{g}_{D_b}, \bm{g}_D } - 
    \hugA{ \bm{g}_{D_b}, \bar{\bm{g}}_D}
    } 
   - \frac{\eta^2}{2} 
    \hugP{
    \mathbb{E}\hugB{\bar{\bm{g}}_B^T \bH_{\ell}^b \bar{\bm{g}}_B}  - 
    \mathbb{E}\hugB{\bm{g}_B^T \bH_{\ell}^b \bm{g}_B}
   } \\
\label{eq:24c}   
         & =\eta \langle \bm{g}_{D_a} -\bm{g}_{D_b}, \bm{g}_D-\bar{\bm{g}}_D \rangle  
         + \frac{\eta^2}{2} \hugP{\EE\hugB{ \bar{\bm{g}}_{B}^T \bm{H}_\ell^a \bar{\bm{g}}_B}  
                          - \EE\hugB{\bm{g}_B^T \bm{H}_\ell^a \bm{g}_B}} \\
\notag &\hspace{10pt}
         - \frac{\eta^2}{2} \hugP{\EE\hugB{ \bar{\bm{g}}_{B}^T \bm{H}_\ell^b \bar{\bm{g}}_B}  
                          - \EE\hugB{\bm{g}_B^T \bm{H}_\ell^b \bm{g}_B}} \\
\label{eq:24d}
        & \geq \eta \langle \bm{g}_{D_a} -\bm{g}_{D_b}, \bm{g}_D-\bar{\bm{g}}_D \rangle  - \frac{\eta^2}{2}\beta_a  \|\bm{g}_D\|^2  - \frac{\eta^2}{2} \beta_b C^2  \\
\label{eq:24e}
        &\geq  \eta \langle \bm{g}_{D_a} -\bm{g}_{D_b}, \bm{g}_D-\bar{\bm{g}}_D \rangle - \frac{\eta^2}{2} \max_{z \in \cA} \beta_z \big(  \|\bm{g}_D\|^2 + C^2 \big) \\
\label{eq:24f}
        & \geq \eta \left( \langle \bm{g}_{D_a} -\bm{g}_{D_b}, \bm{g}_D-\bar{\bm{g}}_D \rangle -\frac{1}{2} ( \|\bm{g}_D\|^2 + C^2)  \right),
\end{align}
\end{subequations}
where the inequality \eqref{eq:24d} follows as a consequence of Lemma \ref{bound_4}, 
and the inequality \eqref{eq:24f} since $\eta \leq \frac{1}{\max_{a \in \cA} \beta_a}$.
\end{proof}

\begin{proof}[Proof of Theorem \ref{thm:grad_clip}]
We want to show that $R^a_{clip} > R^b_{clip}$ given Equation \eqref{eq:5}.
Since, by Lemma \ref{lemma3} the difference $R^a_{clip} - R^b_{clip}$ 
is lower bounded -- see Equation \eqref{eq:23}, the following shows that 
the right hand side of Equation \eqref{eq:23} is positive, that is:
\begin{equation}
\label{eq:24}
  \left\langle \bm{g}_{D_a} -\bm{g}_{D_b}, \bm{g}_D - \bm{\bar{\bm{g}}}_D \right\rangle 
  - \frac{1}{2} \left( \|\bm{g}_D\|^2 + C^2 \right) > 0.
\end{equation}

First, observe that the gradients at the population level can be expressed as
a combination of the gradients of the two groups $a$ and $b$ in the dataset:
$\bm{g}_D =  p_a \bm{g}_{D_a} + p_b \bm{g}_{D_b}$ and $\bar{\bm{g}} = p_a \bar{\bm{g}}_{D_a} + p_b \bar{\bm{g}}_{D_b}$. 

By algebraic manipulation, and the above, Equation \eqref{eq:24} can thus be expressed as:
\begin{subequations}
\begin{align}
  \eqref{eq:24}
  &=  \langle \bm{g}_{D_a} -\bm{g}_{D_b},  p_a \bm{g}_{D_a} + p_b \bm{g}_{D_b} - p_a \bar{\bm{g}}_{D_a} -p_b \bar{\bm{g}}_{D_b} \rangle - \frac{1}{2}\big( \|\bm{g}_{D_a} p_a + \bm{g}_{D_b} p_b\|^2  + C^2\big)\\
  &= ( p_a \|\bm{g}_{D_a}\|^2   + p_b \bm{g}_{D_a}^T\bm{g}_{D_b}  - p_a \bm{g}_{D_a}^T \bar{\bm{g}}_{D_a} - p_b \bm{g}_{D_a}^T \bar{\bm{g}}_{D_b} - p_a \bm{g}_{D_b}^T \bm{g}_{D_a} - p_b \|\bm{g}_{D_b}\|^2 \\
  & + p_a \bm{g}_{D_b}^T \bar{\bm{g}}_{D_a} + p_b \bm{g}_{D_b}^T\bar{\bm{g}}_{D_b} -\frac{1}{2}\big( p^2_a \|\bm{g}_{D_a}\|^2 + 2p_a p_b \bm{g}_{D_a} \bm{g}_{D_b} + p^2_b \|\bm{g}_{D_b}\|^2 +  C^2\big). \notag
\end{align}
\end{subequations}
Noting that for any vector $\bm{x}, \bm{y}$ the following inequality hold:
 $\bm{x}^T \bm{y} \geq -\|\bm{x}\| \|\bm{y}\|$, all the inner products 
 in the above expression can be replaced by their lower bounds:

\begin{subequations}
\begin{align}
  \eqref{eq:24}
  &\geq \|\bm{g}_{D_a}\|\hugP{ \|\bm{g}_{D_a}\| p_a (1- \frac{p_a}{2}) - p_b  \|\bm{g}_{D_b}\| -p_a C  - p_b  C  -p_a \|\bm{g}_{D_b}\|} \\
  & \hspace{12pt} - \|\bm{g}_{D_b}\| \hugP{ \|\bm{g}_{D_b}\| p_b (1 + \frac{p_b}{2}) + p_a C +p_b C   } -\frac{1}{2}C^2 \notag \\
  & = \|\bm{g}_{D_a}\|\hugP{ \|\bm{g}_{D_a}\| p_a (1- \frac{p_a}{2}) - (p_b +p_a) (  \|\bm{g}_{D_b}\| +C) } \\
  & \hspace{12pt} - \|\bm{g}_{D_b}\| \hugP{ \|\bm{g}_{D_b}\| p_b (1 + \frac{p_b}{2}) + (p_a +p_b) C   } -\frac{1}{2}C^2 \notag  \\
\label{eq:28c}
  & = \|\bm{g}_{D_a}\|\hugP{ \|\bm{g}_{D_a}\| p_a (1- \frac{p_a}{2}) -   \|\bm{g}_{D_b}\| - C) }
  - \|\bm{g}_{D_b}\| \hugP{ \|\bm{g}_{D_b}\| p_b (1 + \frac{p_b}{2}) + C   } -\frac{1}{2}C^2 
\end{align}
\end{subequations}
where the last equality is because $p_a +p_b=1$, by assumption of the 
dataset having exactly two groups.

By theorem assumption, 
$\|\bm{g}_{D_a}\| p_a (1- \frac{p_a}{2}) \geq \frac{5}{2}C + \|\bm{g}_{D_b}\| (1 + p_b +\frac{p_b^2}{2})$.
It follows that $\|\bm{g}_{D_a}\| > \|\bm{g}_{D_b}\| $ and $\|\bm{g}_{D_a}\| > C $. 
Combined with Equation \eqref{eq:28c} it follows that:
\begin{subequations}
\begin{align}
 \eqref{eq:28c} &= 
 \|\bm{g}_{D_a}\| \hugP{  \|\bm{g}_{D_a}\| p_a (1- \frac{p_a}{2}) -   \|\bm{g}_{D_b}\| - C -  \|\bm{g}_{D_b}\| p_b (1 + \frac{p_b}{2}) -  C} -\frac{1}{2} C^2\\
& = \|\bm{g}_{D_a}\| \hugP{ \|\bm{g}_{D_a}\| p_a (1- \frac{p_a}{2}) - 2C
    -  \|\bm{g}_{D_b}\| ( 1 + p_b  + \frac{p^2_b}{2})} -\frac{1}{2} C^2 \\
& \geq \|\bm{g}_{D_a}\| \frac{C}{2} - \frac{1}{2}C^2 \\
& > 0,
  \label{eq_final_proof}
\end{align}
\end{subequations}
where the last equality is because $\|\bm{g}_{D_a} \| > C$. 
\end{proof}

\begin{theorem}
\label{thm:add_noise}
For groups $a,b \in \cA$, $R^{noise}_a > R^{noise}_b$ whenever 
\[ \Tr(\bm{H}^a_{\ell}) > \Tr(\bm{H}^b_{\ell}).\]
\end{theorem}

\begin{proof}
 Suppose $ \Tr(\bm{H}^a_{\ell}) > \Tr(\bm{H}^b_{\ell})$. By definition of $R^{noise}_a$ and $R^{noise}_b$ from  Theorem \ref{thm:taylor} it follows that:
 
 $$R^{noise}_a = \frac{\eta^2}{2} \Tr( \bH_{\ell}^a ) C^2 \sigma^2   > \frac{\eta^2}{2} \Tr( \bH_{\ell}^b ) C^2 \sigma^2 = R^{noise}_b,$$
 which concludes the proof.

\end{proof}

\begin{theorem}
\label{thm:boundary}
Consider a $K$-class classifier $\bm{f}_{\btheta,k}$ ($k \in [K]$). 
For a given sample $X \sim D$, the term
$\hugP{1 - \sum_{k=1}^K \bm{f}_{\btheta, k}^2(X)}$ 
is maximized when
$\bm{f}_{\btheta, k}(X) = \nicefrac{1}{K}$ and minimized when 
$\exists k \in [K]$ s.t.~$\bm{f}_{\btheta, k}(X) = 1$ 
and $\bm{f}_{\btheta, k'} = 0 \ \forall k'\in [K], k' \neq k$.
\end{theorem}

\begin{proof}
Fix an input $X$ of $D$ and denote $y_k = \bm{f}_{\btheta,k}(X) \in [0,1]$. 
Recall that $y_k$ represents the likelihood of the prediction of input
$X$ to be associated with label $k$.

Note that, by Cauchy–Schwarz inequality
\begin{subequations}
\begin{align}
\label{eq:31a}
    1 - \sum_{k=1}^K y_k^2 & \leq 1 -  K \hugP{\frac{\sum_i^K y_k}{K}}^2 \\
\label{eq:31b}
    & = 1 - \frac{1}{K}, 
\end{align}
\end{subequations}
where Equation \eqref{eq:31b} follows since $\sum_i^K y_k(X) = 1$.
The above expression is maximized  when 
\[y_k = \bm{f}_{\theta, k}(X)= \frac{1}{K}.\] 

Additionally, since $y_k \in [0,1]$ it follows that $y_k^2 \leq y_k$. Hence,
\begin{equation}
    1 - \sum_{k=1}^K y_k^2 \geq  1 - \sum_{i=1}^K y_k = 0.
\end{equation}
To hold, the equality above, it must exists $k \in [K]$ such that
 $y_k = \bm{f}_{\theta, k}(X) = 1$ 
 and for any other $k' \in [K]$ with $k' \neq k$, 
 $y_{k'} = \bm{f}_{\theta, k'} = 0$. 
\end{proof}

Given the connection of the term $1 - \sum_{k=1}^K (1 - f^2_{\theta,k}(X))$ and 
the associated (trace of the) Hessian loss $\bm{H}_f$, the result above suggests that 
the trace of the Hessian is minimized (maximized) when the classifier is very 
confident (uncertain) about the prediction of $X \sim D$ , i.e.,~when $X$ is far 
(close) to the decision boundary.

\section{Experimental settings}
\label{sec:experimental_settings}

\paragraph{Datasets} The paper uses the following UCI datasets to support its claims:
\begin{enumerate}
    \item \textbf{Adult} (Income) dataset, where the task is to predict if an 
    individual has low or high income, and the group labels are defined by race: 
    \emph{White} vs \emph{Non-White} \cite{UCIdatasets}.

    \item \textbf{Bank} dataset, where the task is to predict if a user subscribes 
    a term deposit or not and the group labels are defined by age: 
    \emph{people whose age is less than 60 years old vs the rest} \cite{Moro2014ADA}. 
    
    \item \textbf{Wine} dataset, where the task is to predict if a given wine is 
    of  good quality, and the group labels are defined by wine color: 
    \emph{red vs white} \cite{UCIdatasets}. 
    
    \item \textbf{Abalone} dataset, where the task is to predict if a given 
    abalone ring exceeds the median value, and the group labels are 
    defined by gender: \emph{female vs male} \cite{UCIdatasets}. 

    
    \item \textbf{Parkinsons} dataset, where the task is to predict if 
    a patient has total UPDRS score that exceeds the median value, 
    and the group labels are defined by gender: \emph{female vs male} \cite{article}.  

    \item \textbf{Churn} dataset, where the task is to predict if a 
    customer churned or not. The group labels are defined by on gender:
    \emph{female vs male} \cite{IBMdataset}.
       
    \item \textbf{Credit Card} dataset, where the task is to predict if 
    a customer defaults a loan or not. The group labels are defined by gender:
    \emph{female vs male} \cite{creditdataset}. 
    
   \item \textbf{Stroke} dataset, where the task is to predict if a patient 
   have had a stroke based on their physical conditions. 
   The group labels are defined by gender: \emph{female vs male} \cite{Strokedataset}. 
    
\end{enumerate}

All datasets were processed by standardization so each feature has zero mean and unit variance. 

\paragraph{Settings}
For output perturbation, the paper uses a Logistic regression model to obtain the optimal model parameters (we set the regularization parameter $\lambda = 1$) and add Gaussian noise to achieve privacy. 
The standard deviation of the noise required to the mechanism is determined following \citet{balle2018improving}. 

For DP-SGD, the paper uses a neural network with single hidden layer with 
\emph{tanh} activation function for the different datasets. The batch size $|B|$ is fixed to $ 32$ and the learning rate $\eta = 1e-4$. Unless specified we set the clipping bound $C =0.1$ and  noise multiplier $\sigma = 5.0$. The experiments consider 100 runs of DP-SGD with different random seeds for each configuration. We employ the Tensorflow Privacy toolbox to compute the privacy loss $\epsilon$ spent during training.

\smallskip\noindent\textbf{Computing infrastructure} 
All experiments were performed on a cluster equipped with Intel(R) 
Xeon(R) Platinum 8260 CPU @ 2.40GHz and 8GB of RAM.

\smallskip\noindent\textbf{Software and libraries}  
All models and experiments were written in Python 3.7 and  in Pytorch 1.5.0. 

\smallskip\noindent\textbf{Code} 
The code used for this submission is attached as supplemental material. All implementation of the experiments and proposed mitigation solution will be released upon publication.

\section{Additional experiments}
\label{sec:additional_experiments}
\subsection{More on ``Warm up: output perturbation''}

\paragraph{Correlation between Hessian trace and excessive risk} The
 following provides additional empirical support for the claims of
 the main paper: \emph{Groups with larger Hessian trace tend to have
 larger excessive risks} in this subsection.

The experiments in this sub-section use {output perturbation}.
Figure \ref{fig:output_pertb_others} report the excessive risk gap
and Hessian traces for the two groups defined in the datasets
(as described in Section \ref{sec:experimental_settings}. The figure
clearly illustrates that the groups with larger Hessian traces have
larger excessive risk gaps (i.e., experienced more unfairness) under
private output perturbation when compared with the groups with
smaller Hessian traces. These empirical findings are again a strong
support for the claims of Theorem \ref{thm:output}.

\begin{figure}[ht]
  \centering
    \centering
    \includegraphics[width=1.0\linewidth]{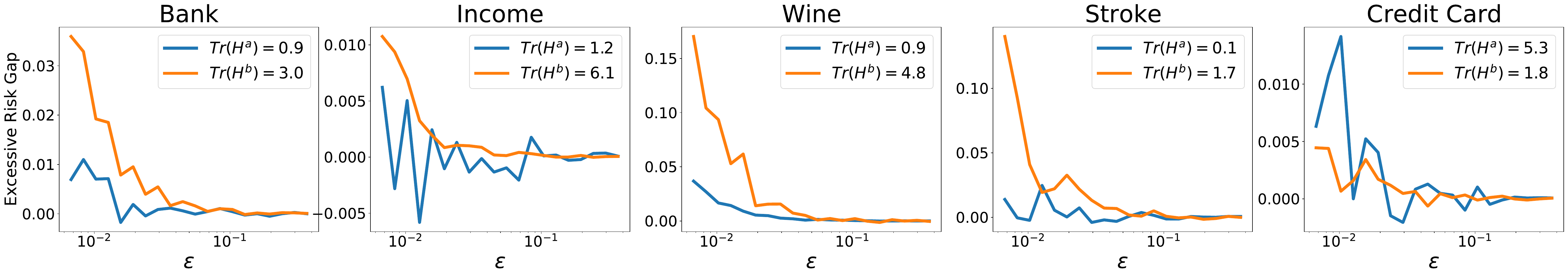}
    \caption{Correlation between  excessive  risk  gap  and Hessian 
    Traces at varying of the privacy loss $\epsilon$.}
    \label{fig:output_pertb_others}
 \end{figure}

 \paragraph{Impact of data normalization by group}
 
The next results provide evidence to support the following claim
raised in Section \ref{sec:optput_pert}: \emph{Given the impact of
gradient norms to unfairness, normalizing data independently for each
group can help improve fairness}. Figure \ref{fig:group_norm} shows
the evolution of the excessive risk $R_a$ and $R_b$ for the dataset
groups during training. The top plots present the results with
standard data normalization (e.g., each sample data is normalized
independently from its group membership) while the bottom plots show
the counterpart results for models trained when the data was
normalized within the group datasets $D_a$ and $D_b$. 
Note that the normalization adopted ensures that the data is $0$-mean
and of unit variance in each group dataset, which is a required 
condition to achieve the desired property.   

The results clearly show that this strategy can not only reduce unfairness,
but also the excessive risk gaps. 
  
\begin{figure}[ht]
  \centering
    \centering
    \includegraphics[width=0.8\linewidth]{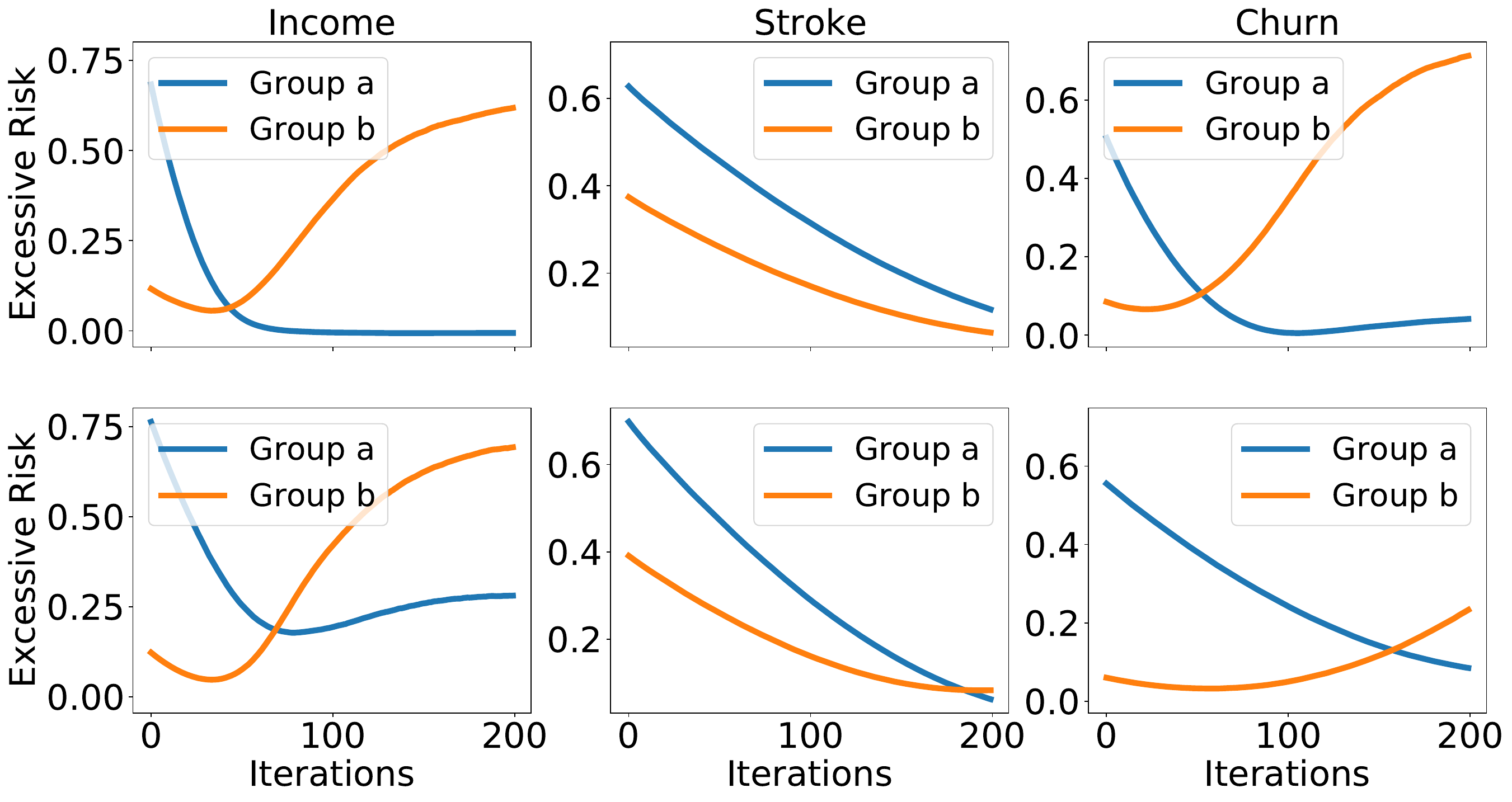}
    \caption{Excessive risk for each group without group normalization 
    (top) and  with group normalization (bottom).}
    \label{fig:group_norm}
 \end{figure}

\subsection{More on ``Why gradient clipping causes unfairness?''}
This section provides additional empirical evidence to support the claim made in Section \ref{sec:clipping} specifying the three direct factors influencing the clipping effect to the excessive risk: 
{\bf (1)} the Hessian loss, 
{\bf (2)} the gradient values, and 
{\bf (3)} the clipping bound. 
Among these three factors, the gradient values and clipping bound are the dominant ones. 

\paragraph{Impact of gradient values and clipping bound $C$}
Figure \ref{fig:grad_clipping} provides the relation between the gradient norm and the different choices of clipping bounds to the excessive risks. The results are shown for the Abalone, Churn and Credit Card datasets. 
The experiments show that gradient norms reduce as $C$ increases and that the group with larger gradient norms have also larger excessive risk.
Similar results were achieved for other datasets as well (not reported to avoid redundancy).

\begin{figure}[!h]
  \centering
  \begin{subfigure}{\linewidth}
    \centering
    \includegraphics[width=1.0\linewidth]{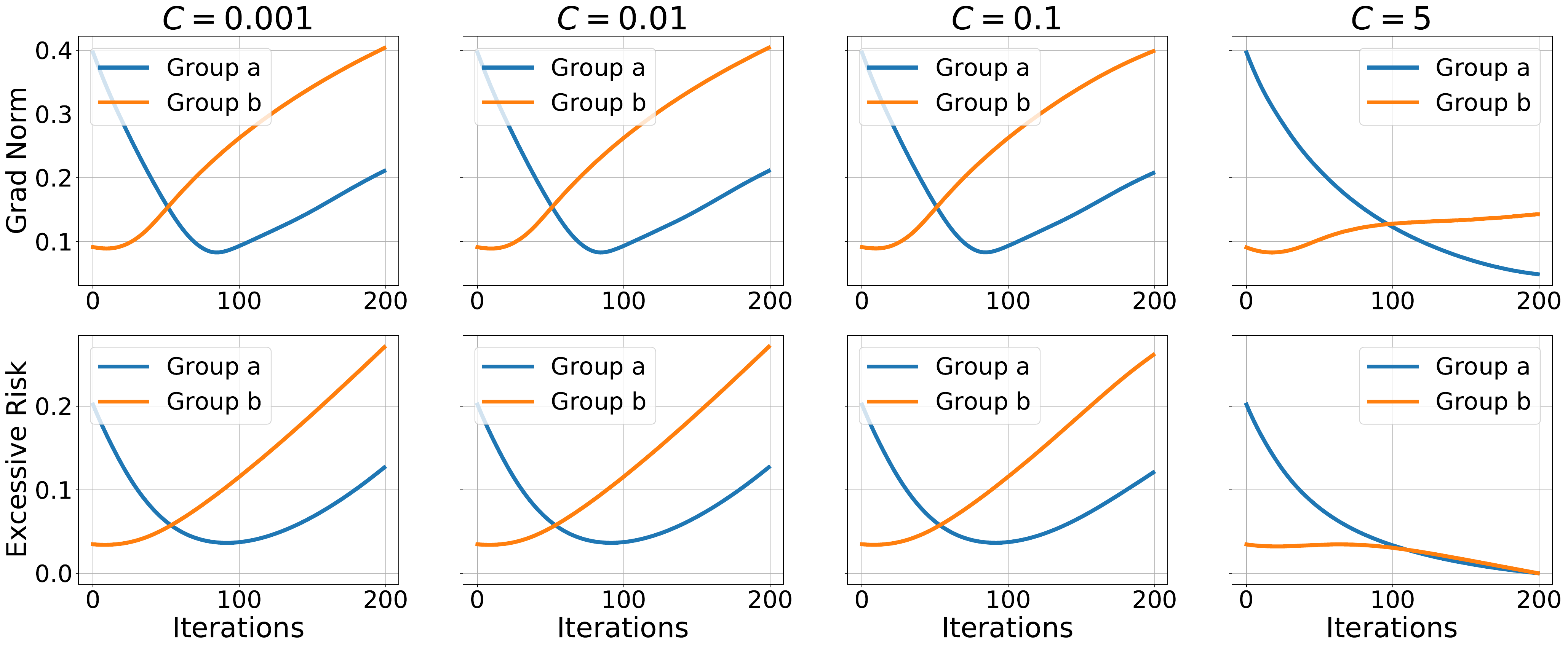}
    \caption{Abalone dataset}
  \end{subfigure}
  
   \begin{subfigure}{\linewidth}
    \centering
    \includegraphics[width=1.0\linewidth]{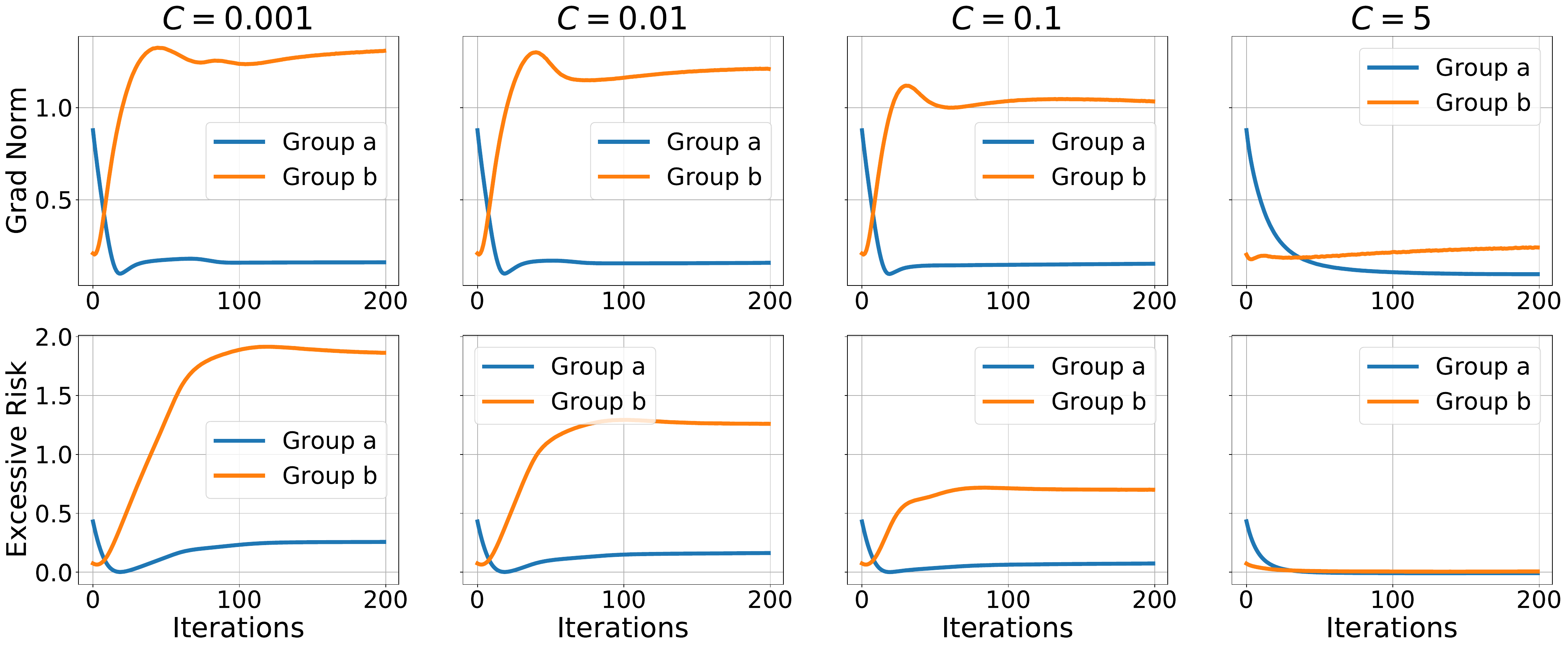}
    \caption{Churn dataset}
  \end{subfigure}
  
    \begin{subfigure}{\linewidth}
    \centering
    \includegraphics[width=1.0\linewidth]{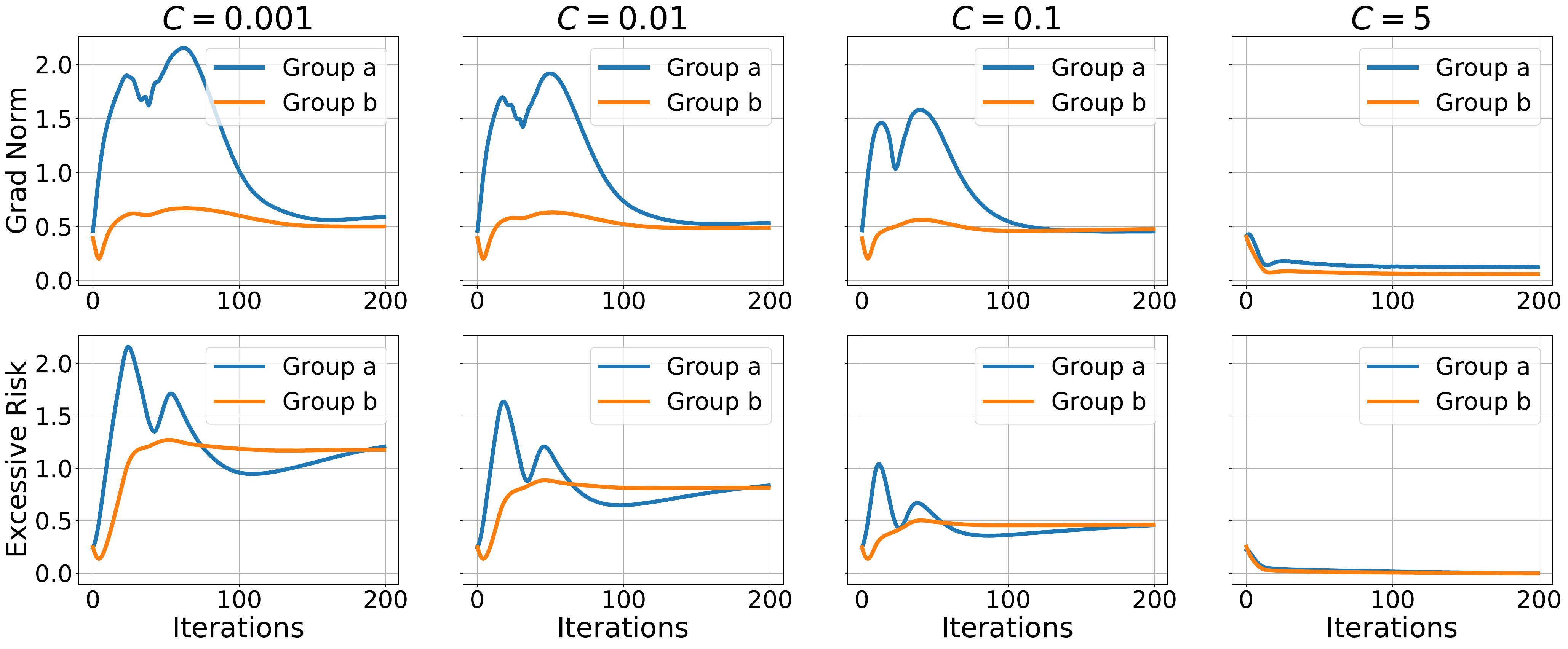}
    \caption{Churn dataset}
  \end{subfigure}
 
  \caption{Impact of gradient clipping with different clipping bound values 
  $C$ to the excessive risk.} 
  \label{fig:grad_clipping}
\end{figure}

\paragraph{The Hessian loss is a minor impact factor to the excessive risk.}
As showed in the main text, the excessive risk associated to the gradient clipping for a particular group $a \in \cA$ can be decomposed as:
\begin{equation}
    R^{clip}_a = \eta \hugP{
    \hugA{ \bm{g}_{D_a}, \bm{g}_D } - 
    \hugA{ \bm{g}_{D_a}, \bar{\bm{g}}_D}
    } 
   + \frac{\eta^2}{2} 
    \hugP{
    \mathbb{E}\hugB{\bar{\bm{g}}_B^T \bH_{\ell}^a \bar{\bm{g}}_B}  - 
    \mathbb{E}\hugB{\bm{g}_B^T \bH_{\ell}^a \bm{g}_B}
   }
\end{equation}

Denote $\psi_a =\hugP{
    \mathbb{E}\hugB{\bar{\bm{g}}_B^T \bH_{\ell}^a \bar{\bm{g}}_B}  - 
    \mathbb{E}\hugB{\bm{g}_B^T \bH_{\ell}^a \bm{g}_B}
   }$. 
This quantity clearly depends on the Hessian loss $\bH_{\ell}^a$. However, under the assumptions in Theorem \ref{thm:grad_clip}: convexity and smoothness of the loss function and the magnitude of the learning rate (i.e., that is small enough), the term $\psi_a$ will be a negligible component in $R^{clip}_a$. 

While this is evident under those assumption, our empirical analysis has reported a similar behavior for loss function for which those conditions do not generally apply. 
In the following experiment we run DP-SGD on a neural network with single hidden layer and tracked the values of $R^{clip}_a$ and $\psi_a$ for each group $a\in \cA$ during private training. 
These values are reported in Figure \ref{fig:hessian_loss_minor} for different datasets. It can be seen that the components $\psi_a$ (dotted lines) constitute a negligible amount to the excessive risk under gradient clipping $R^{clip}_a$.  

\begin{figure}[!h]
  \centering
    \centering
    \includegraphics[width=0.66\linewidth]{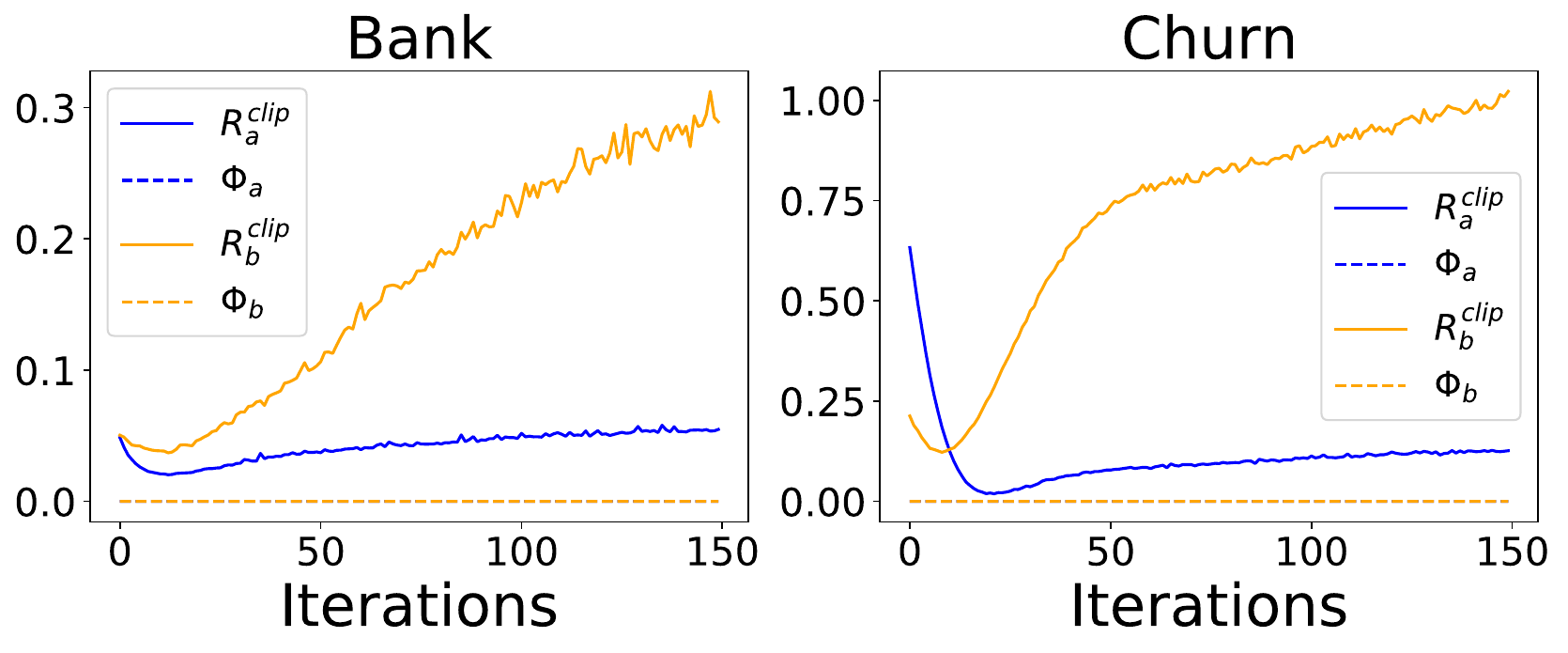}\\
    \includegraphics[width=0.66\linewidth]{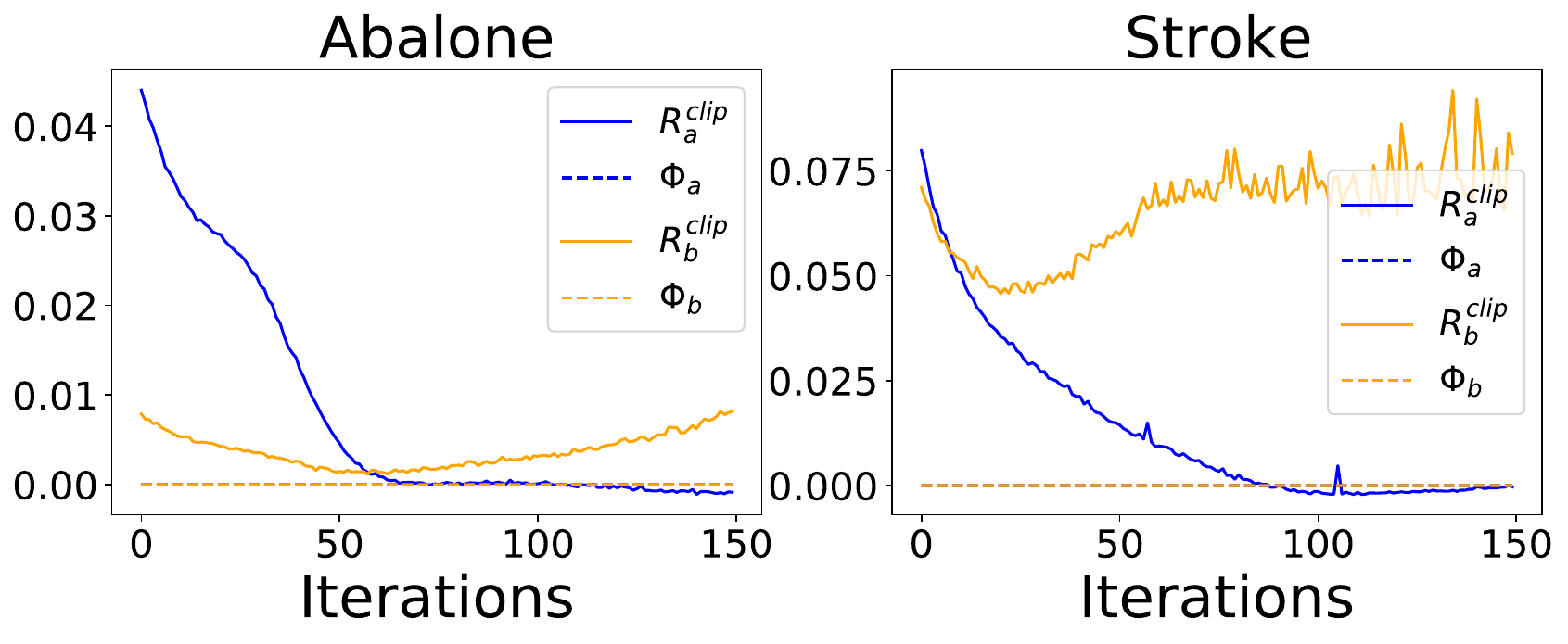}
    \caption{Values of $R^{clip}_a$ and $\psi_a$ during private training 
    for a neural network classifier. }
    \label{fig:hessian_loss_minor}
 \end{figure}

\paragraph{Relative group data size is a minor impact factor  to the excessive risk.}

Section \ref{sec:clipping} also observed that the relative group data 
size, $\nicefrac{p_b}{p_a}$ for two groups $a, b \in \cA$ had a minor impact
on unfairness. 
Figure \ref{fig:pct_group} provides empirical evidence to support 
this observation. It shows the effects of varying the relative group data
$\nicefrac{p_b}{p_a}$ to the gradient norms (top rows) and excessive risk (bottom rows)
in three datasets: Abalone, Bank, and Income.
The different relative group data ratios were obtained through subsampling.
Notice that changing the relative group sizes does not result in a 
noticeable effect in the group gradient norms and excessive risk. 
These experiments demonstrate that the relative group data size 
might play a minor role in affecting unfairness. 

These observation are also in alignment with the those raised by 
\citet{farrand2020neither}, who showed that the disparate impact of 
DP on model accuracy is not limited to highly imbalanced data and can
occur in situations where the groups are slightly imbalanced.

\begin{figure}[!h]
  \centering
  \begin{subfigure}{\linewidth}
    \centering
    \includegraphics[width=0.8\linewidth]{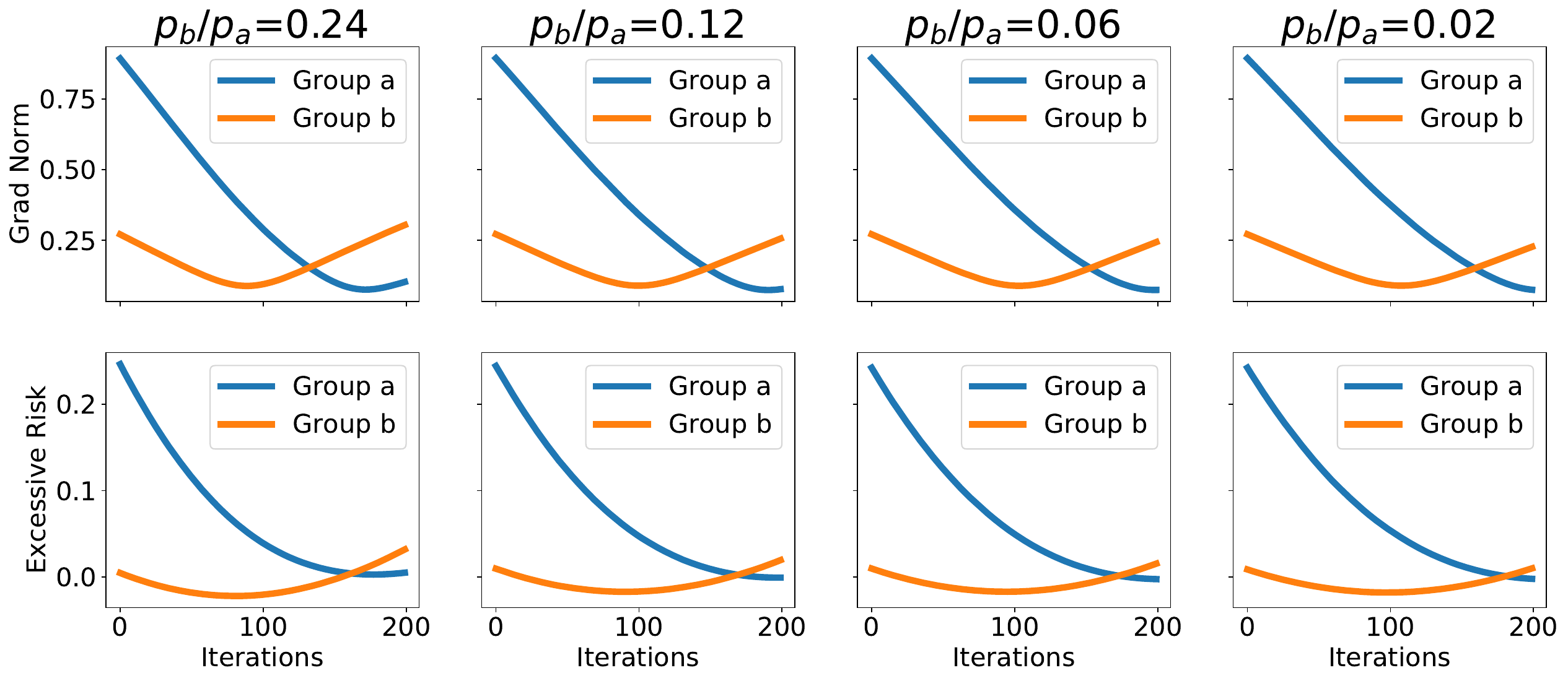}
    \caption{Abalone dataset}
  \end{subfigure}
  
   \begin{subfigure}{\linewidth}
    \centering
    \includegraphics[width=0.8\linewidth]{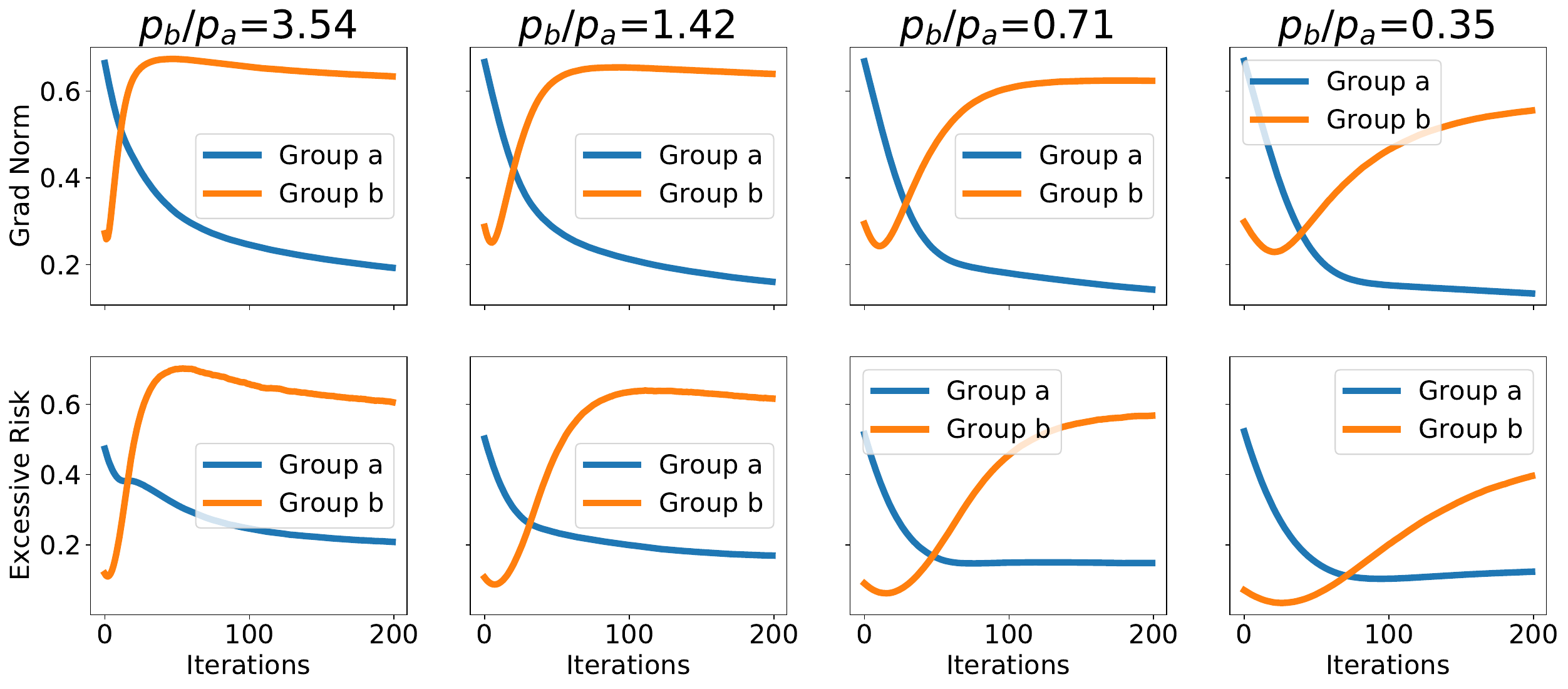}
    \caption{Bank dataset}
  \end{subfigure}
  
    \begin{subfigure}{\linewidth}
    \centering
    \includegraphics[width=0.8\linewidth]{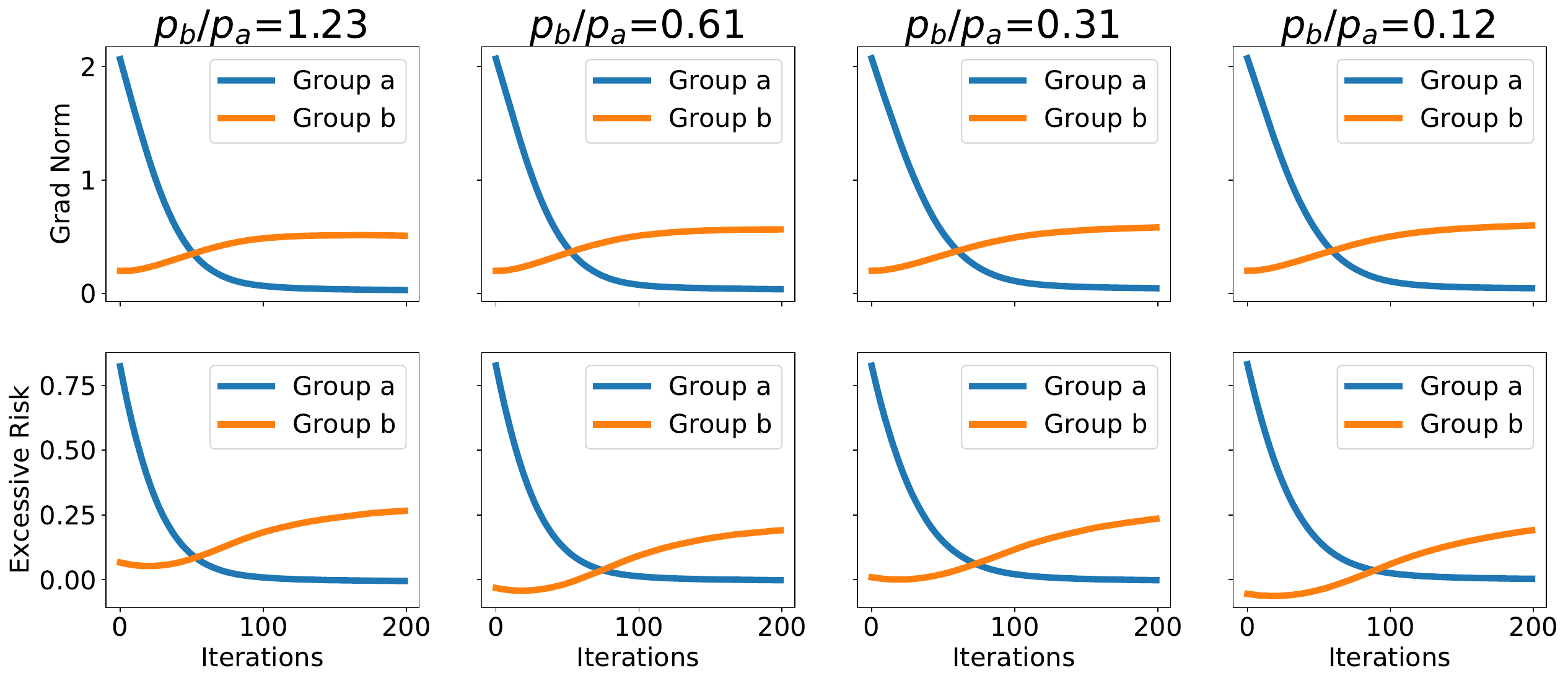}
    \caption{Income dataset}
  \end{subfigure}
 
  \caption{Impact of the relative group data size towards unfairness 
  under DP-SGD (with $C=0.1, \sigma = 5.0$).} 
  \label{fig:pct_group}
\end{figure}

\subsection{More on ``Why noise addition causes unfairness?''}


Figure \ref{fig:corr} illustrates the connection between the trace of 
the Hessian of the loss function at some sample $X \in D$ and its distance 
to the decision boundary. 
The figure clearly show that the closest (father) is a sample $X$ to the 
decision boundary, the larger (smaller) is the associated Hessian trace value 
$\Tr(\bm{H}_{\ell}^X)$.
The experiments are reported for datasets Parkinson, Stroke, Wine, and 
Churn, but once again they extend to other datasets as well. 

\begin{figure}[!h]
     \centering
     \begin{subfigure}[b]{0.35\textwidth}
         \centering
         \includegraphics[width=\textwidth]{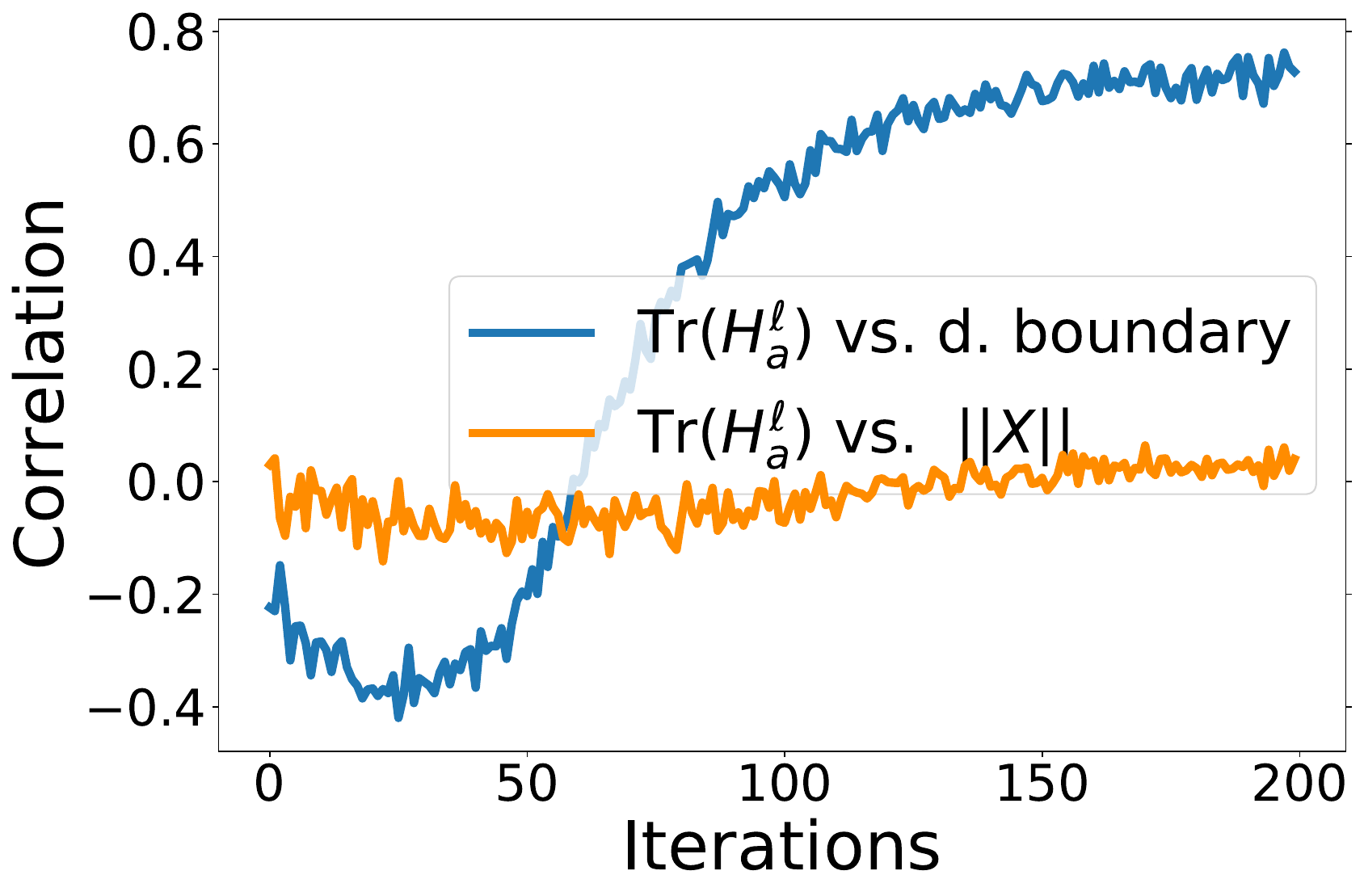}
         \caption{Parkinson dataset}
     \end{subfigure}
     \begin{subfigure}[b]{0.35\textwidth}
         \centering
         \includegraphics[width=\textwidth]{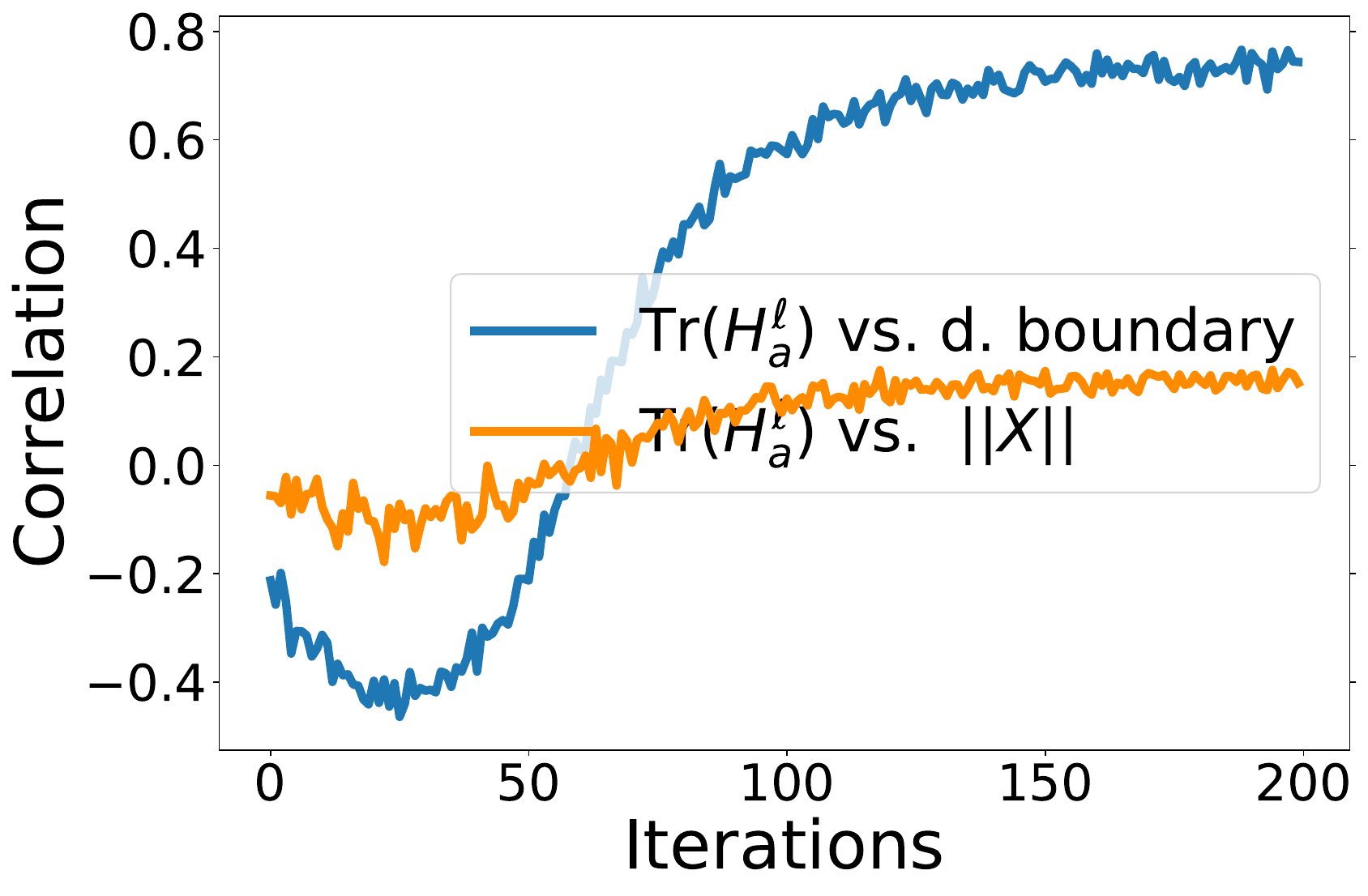}
         \caption{Stroke dataset}
     \end{subfigure}
     \begin{subfigure}[b]{0.35\textwidth}
         \centering
         \includegraphics[width=\textwidth]{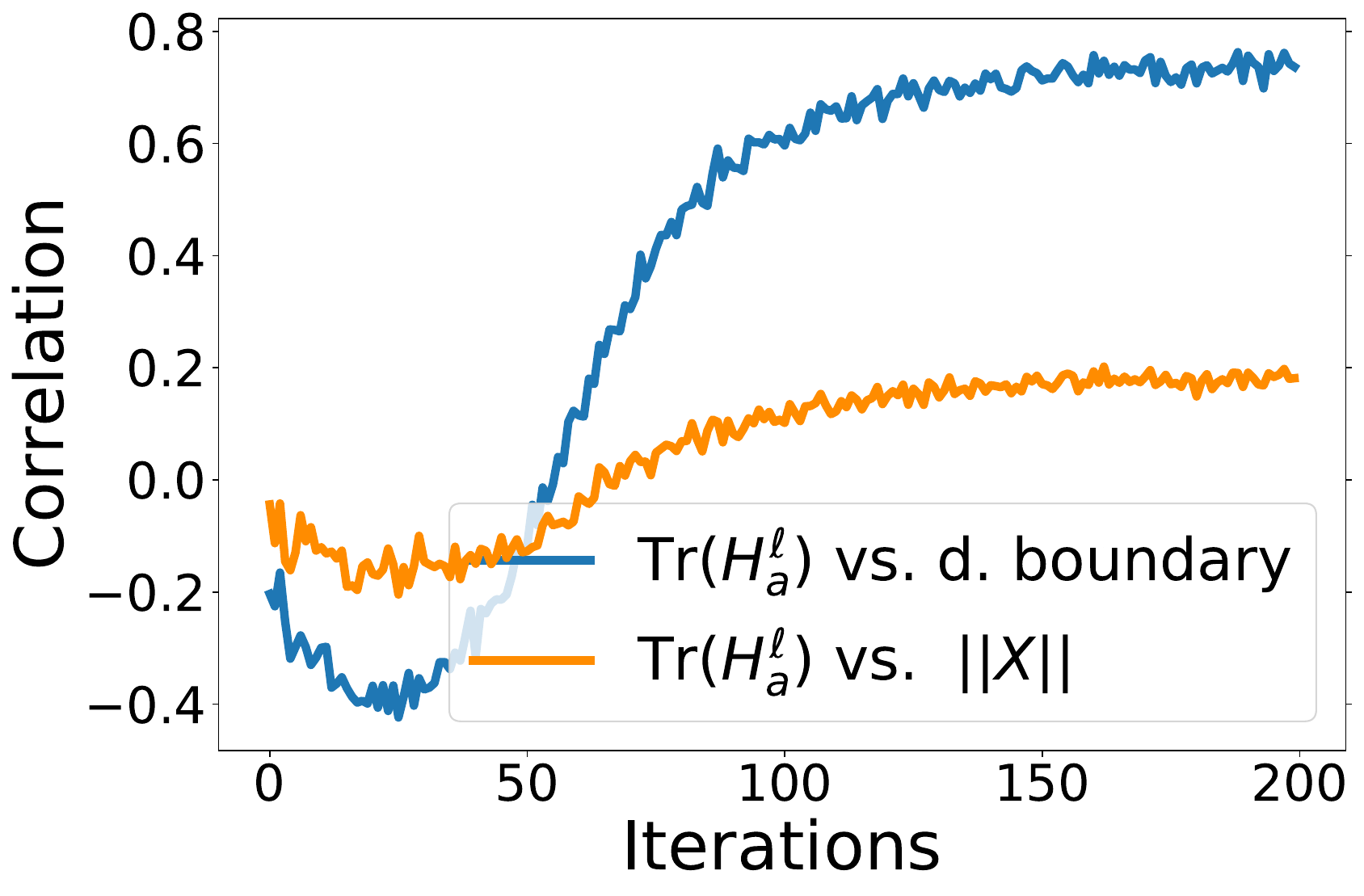}
         \caption{Wine dataset}
     \end{subfigure}
    \begin{subfigure}[b]{0.35\textwidth}
         \centering
         \includegraphics[width=\textwidth]{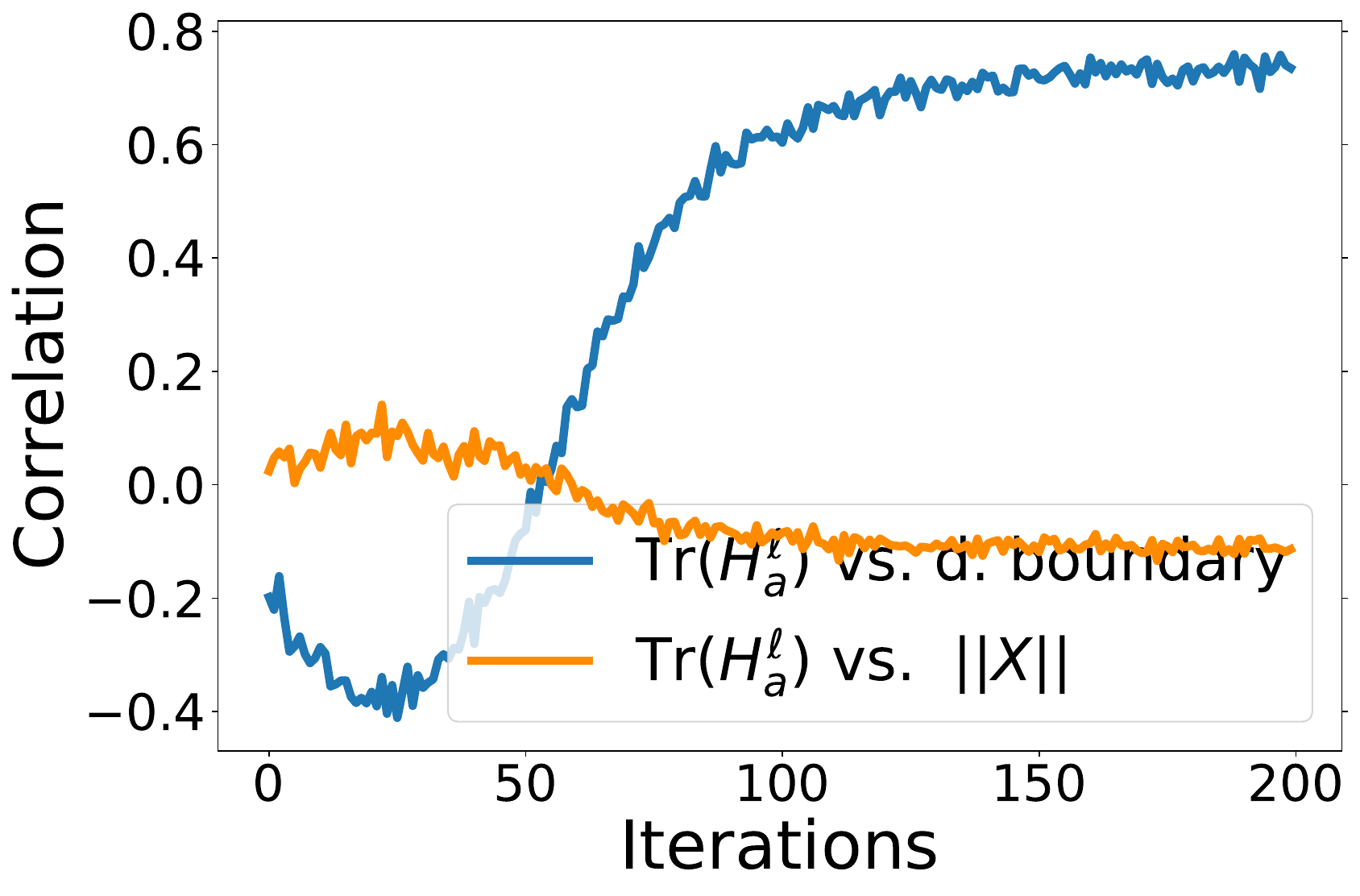}
         \caption{Churn dataset}
     \end{subfigure}
    \caption{Correlation between the trace of the Hessian of the loss 
    function for a data sample $X$ with its distance to the decision boundary 
    (dark colors) and input norm (light colors).}
    \label{fig:corr}
\end{figure}

\subsection{More on mitigation solutions }
Next, this section demonstrates the benefits of the proposed mitigation 
solution on additional datasets.  
Figure \ref{fig:mitigation_others} illustrates the excessive risk for 
each group in the reported datasets (recall that better fairness is 
achieved when the excessive risk curves values are small and similar)
at varying of the privacy parameter $\epsilon$ (i.e., the excessive 
risk is tracked during private training). 

The leftmost column in each sub-figure present the results for the baseline model, 
which runs DP-SGD without the proposed fairness-mitigating constraints. 
Observe the positive effects in reducing the inequality between the excessive 
risks between the groups when the solution activates both $\gamma_1$ 
(which regulates the component associated with $R^\clip$) and $\gamma_2$ (which 
regulates the component associated with $R^\noise$). 
In the reported experiments hyper-parameters $\gamma_1 = 1, \gamma_2 =1$ 
were found to be good values for all our benchmark datasets. 
Smaller $\gamma_1$ and $\gamma_2$ values may not reduce unfairness. 
Likewise, large values could even exacerbate unfairness. 
Using the above setting, the proposed mitigation solution 
was able not only to reduce unfairness in 6 out 8 cases studied, but
also to increase the utility of the private models. 

Once again, we mention that the design of optimal hyper-parameters is 
an interesting open challenge.

\begin{figure}[!h]
  \centering
  \begin{subfigure}{\linewidth}
    \centering
    \includegraphics[width=0.85\linewidth]{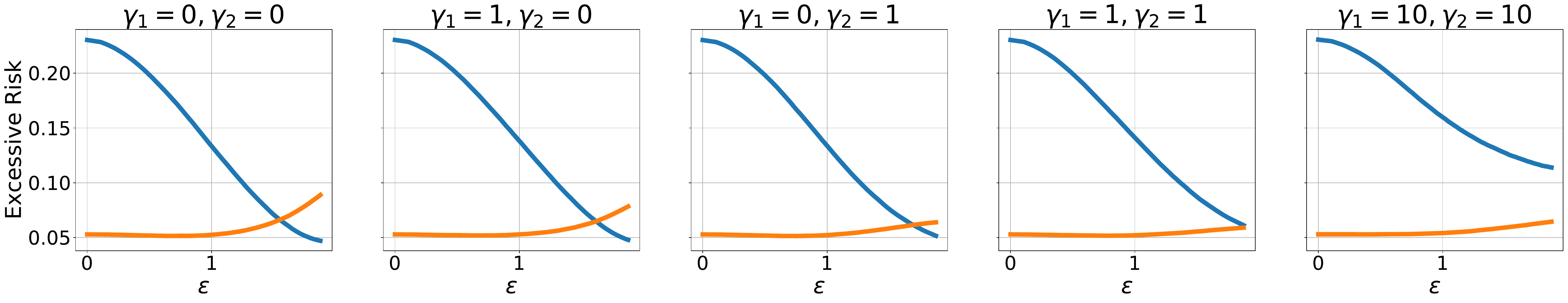}
    \caption{Abalone dataset}
  \end{subfigure}

  \begin{subfigure}{\linewidth}
    \centering
    \includegraphics[width=0.85\linewidth]{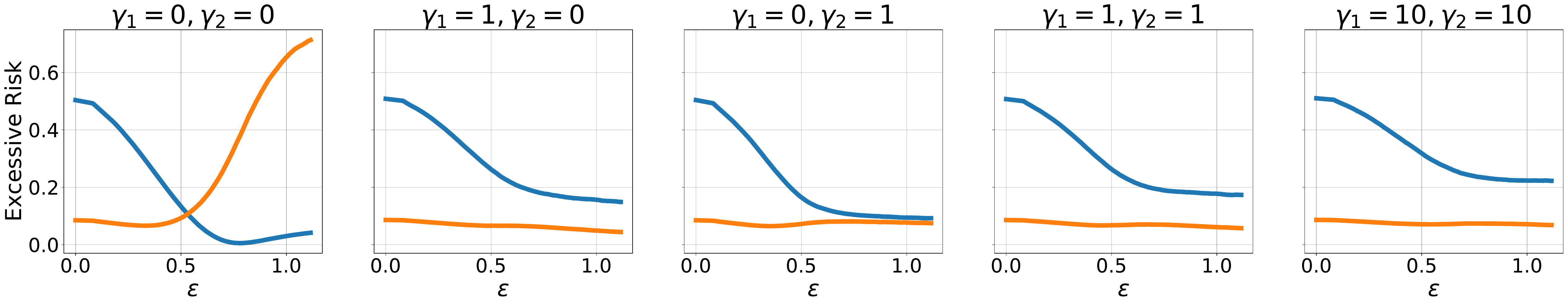}
    \caption{Churn dataset}
  \end{subfigure}  
  
  \begin{subfigure}{\linewidth}
    \centering
    \includegraphics[width=0.85\linewidth]{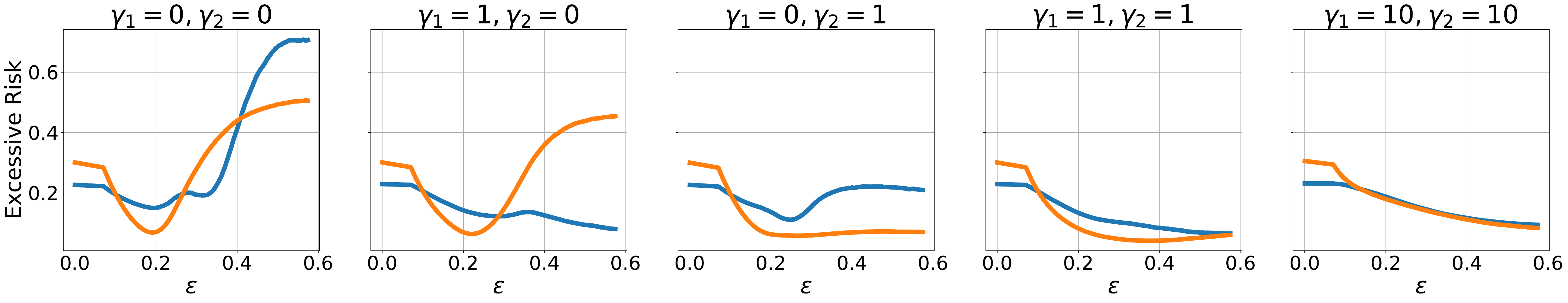}
    \caption{Credit Card dataset}
  \end{subfigure}  
  
  \begin{subfigure}{\linewidth}
    \centering
    \includegraphics[width=0.85\linewidth]{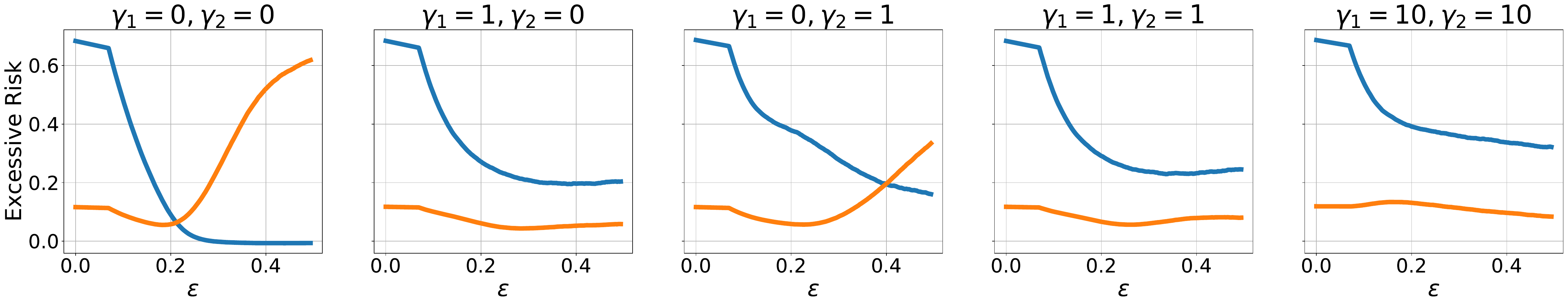}
    \caption{Income dataset}
      \end{subfigure}  
    
    \begin{subfigure}{\linewidth}
    \centering
    \includegraphics[width=0.85\linewidth]{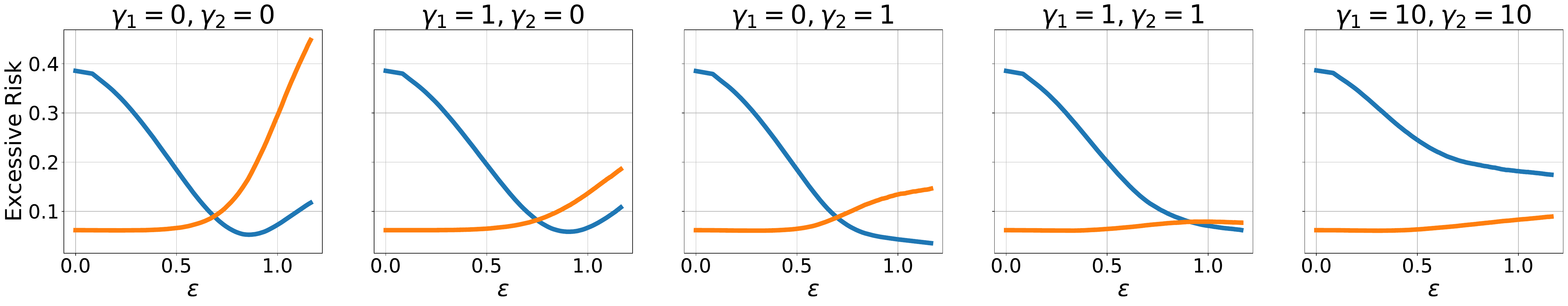}
    \caption{Wine dataset}

\begin{subfigure}{\linewidth}
    \centering
    \includegraphics[width=0.85\linewidth]{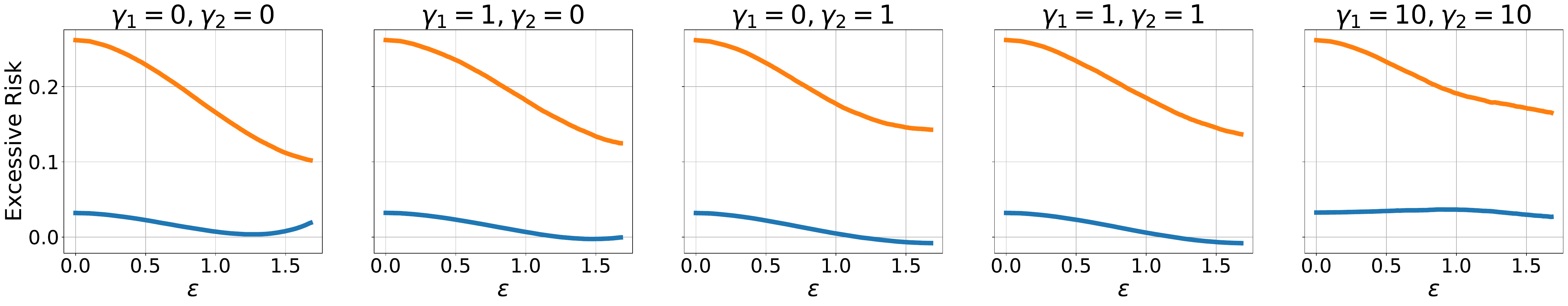}
    \caption{Parkinsons dataset}
  \end{subfigure}  
  
  \begin{subfigure}{\linewidth}
    \centering
    \includegraphics[width=0.85\linewidth]{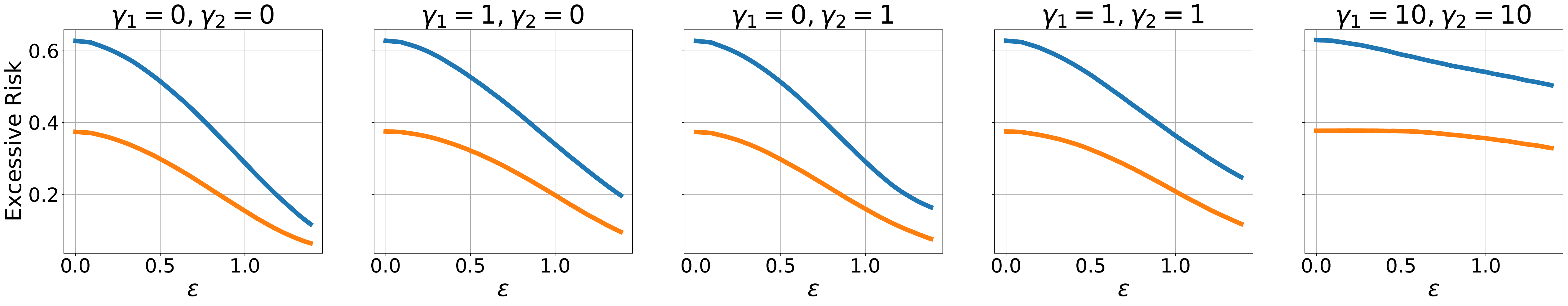}
    \caption{Stroke dataset}
  \end{subfigure}  
  
  \end{subfigure}  
  
  \caption{Mitigating solution: Excessive risk at varying of the 
  privacy loss $\epsilon$ for different $\gamma_1$ and $\gamma_2$.} 
  \label{fig:mitigation_others}
\end{figure}

\section{Additional examples}
\label{sec:additional_examples}
\subsection{More on gradient and Hessian loss of neural networks}
This section focuses on two tasks: 
The first is to demonstrate the connection between the gradient norm 
$\|\bm{g}_X\|$ for some input $X$ with its input norm $\|X\|$. 
The second is to demonstrate the relation between the trace of the 
Hessian loss at a sample $X$ with input norm $\|X\|$ and the closeness 
of $X$ to the decision boundary.
We do so by providing a derivation of the gradients and the Hessian 
trace of a neural networks with one hidden layer. 

\paragraph{Settings} 
Consider a neural network model $f_{\btheta}
 (X) \defeq \textsl{softmax} \left(\btheta_1^T \sigma(\btheta^T_2 X) \right)$ 
 where $X \in \RR^d, \theta_2 \in \RR^{d \times H}, \theta_1 \in \RR^{H \times K}$ 
 and the cross-entropy loss $\ell(f_{\btheta}(X),Y) = -\sum_{k=1}^K Y_i \log 
 \bf{f}_{\btheta, k}(X)$ where $K$ is the number of classes, and $\sigma(\cdot)$ 
 is the a proper activation function, e.g.~ a sigmoid function.
 Let $\bm{O} = \sigma(\btheta^T_2 X) \in \RR^H$
 be the vector $(O_1, \ldots, O_H)$ of $H$ hidden nodes of the network. 
 Denote with $h_j = \sum_i \theta_{ji} X_i$ as the $j$-th hidden unit 
 before the activation function. 
 Next, denote $\btheta_{1,j,k} \in \RR$ as the weight parameter that connects 
 the $j$-th hidden unit $h_j$ with the $k$-th output unit $f_k$ and
 $\theta_{2,i,j} \in \RR$ as the weight parameter that connects the $i$-th
 input $X_i$ unit with the $j$-th hidden unit $h_j$.

\paragraph{Gradients Norm}
First notice that we can decompose the gradients norm of this neural network into two layers as follows:
\begin{equation}
    \| \nabla_{\btheta}\ell(f_{\btheta}(X),Y) \|^2 =  \| \nabla_{\btheta_1}\ell(f_{\btheta}(X),Y) \|^2 + \| \nabla_{\btheta_2}\ell(f_{\btheta}(X),Y) \|^2.
\end{equation}
We will show that $\nabla_{\btheta_2}\ell(f_{\btheta}(X),Y) \| \propto  \|X\|.$

Notice that: 
$$
\|\nabla_{\btheta_2}\ell(f_{\btheta}(X),Y) \|^2 = \sum_{i,j} \| \nabla_{\btheta_{2,i,j}}\ell(f_{\btheta}(X),Y) \|^2.
$$

Applying, Equation (14) from \citet{gradient_formula}, it follows that: 
\begin{equation}
   \nabla_{\theta_{2,i,j}}\ell(f_{\btheta}(X),Y) = 
   \sum_{k=1}^K \left(Y_k - \bm{f}_{\theta,k}(X) \right) \, 
   \theta_{1,j,k}\left(O_j(1 - O_j) \right) X_i,
\end{equation}
which highlights the dependency of the gradient norm 
$\| \nabla_{\btheta_2}\ell(f_{\btheta}(X),Y) \|$ and the input norm $\| X\|^2$.

\paragraph{Hessian trace}
For the connections between the Hessian trace of the loss function at a sample $X$ with the closeness of $X$ to the decision 
boundary and the input norm $\|X\|$, the analysis follows the derivation 
provided by \citet{bishop1992exact}. First, notice that:
\begin{equation}
   \Tr(\bm{H}^{X}_{\ell}) = \Tr(\nabla^2_{\btheta_1}\ell(f_{\btheta}(X),Y)) + \Tr(\nabla^2_{\btheta_2}\ell(f_{\btheta}(X),Y))
\end{equation}
The following shows that:
\begin{enumerate}[leftmargin=*, parsep=0pt, itemsep=2pt, topsep=-4pt]
\item $\Tr\left(\nabla^2_{\btheta_2}\ell(f_{\btheta}(X),Y)\right) \propto \hugP{1 - \sum_{k=1}^K \bm{f}_{\btheta, k}^2(X)} $  
\item $\Tr\left(\nabla^2_{\btheta_1}\ell(f_{\btheta}(X),Y)\right) \propto \| X \|^2$. 
\end{enumerate}

The former follows from Equation (26) of \citet{bishop1992exact}, since:
\begin{equation}
    \nabla^2_{\btheta_{1,j,k}}\ell(f_{\btheta}(X),Y))  =   f_k(1- f_k)O^2_j,
\end{equation}
and thus,
$$\Tr(\nabla^2_{\btheta_1}\ell(f_{\btheta}(X),Y)) = 
\sum_{j=1}^H \sum_{k=1}^K f_k (1-f_k) O^2_j = \sum_{j=1}^H \big(\sum_{k=1}^K f_k - \sum_{k=1}^K f^2_k \big) O^2_j = (1- \sum_{k=1}^K f^2_k ) \sum_{j=1}^H O^2_j.$$

The above shows the connection between the trace of Hessian loss at a sample $X$ 
for the second layer of the neural network and the quantity 
$1 - \sum_{k=1}^K f^2_k(X)$ which measures how close is the sample $X$ to 
the decision boundary. This result relates with Theorem \ref{thm:boundary}. 

Regarding point (2), by applying Equation (27) of  \cite{bishop1992exact} we obtain:
\begin{equation}
    \nabla^2_{\btheta_{2,i,j}}\ell(f_{\btheta}(X),Y))  = X^2_i \Gamma_j,
\end{equation}
where $\Gamma_j = \sigma''(h_j) \sum_{k=1}^K \theta_{2,j,k}(Y_k - f_k) 
     + \sigma'(h_j)^2 \sum_{k=1}^K \btheta_{2,j,k}^2 f_k (1-f_k) $, 
where $\sigma'$ and $\sigma''$ are, respectively, the first and second 
derivative of the activation $\sigma$ with respect to the hidden node $h_j$. 

Thus:
$$\Tr(\nabla^2_{\btheta_2}\ell(f_{\btheta}(X),Y)) = \sum_{j=1}^H 
\sum_{i=1}^d \nabla^2_{\btheta_{2,i,j}} \ell(f_{\btheta}(X),Y) 
= \sum_{j=1}^H \left( \sum_{i=1}^d X^2_i\right) \Gamma_j \propto  \|X\|^2, $$

which shows the dependency of the trace of the Hessian of the loss function in the first layer at sample $X$ and the data input norm.





\end{document}